\documentclass[journal]{IEEEtran}
\usepackage{bm}
\usepackage{amsmath}
\usepackage{graphicx}

\usepackage{algorithm}
\usepackage{algorithmic}
\usepackage{mathtools}
\usepackage{cite}
\usepackage{caption,subcaption}
 \usepackage[justification=centering]{caption}

\usepackage{float}
\usepackage{tabularx}
\usepackage{multirow,makecell}
\usepackage{tikz}
\usepackage{siunitx}
\usepackage{color}

\newtheorem{lemma}{Lemma}
\newtheorem{proof}{Proof}[section]

\makeatletter
\renewcommand{\p@subfigure}{\thefigure.\space}
\makeatother
\begin{document}
\title{NPSA: Nonorthogonal Principal Skewness Analysis }
\author{Xiurui Geng, Lei Wang
\thanks{X.Geng and L.Wang are with Institute of Electronics,
Chinese Academy of Science, Beijing 100864, China, with the Key Laboratory
Systems, Chinese Academy of Science, Beijing 100864, China, and also
with the University of the Chinese Academy of Sciences, Beijing 100049,
China.(Corresponding authors: Lei Wang, Xiurui Geng. Email address:wanglei179@mails.ucas.ac.cn;gengxr@sina.com.cn.)} }




\maketitle

\begin{abstract}
Principal skewness analysis (PSA) has been introduced for feature extraction  in hyperspectral imagery. As a third-order generalization of principal component analysis (PCA), its solution of searching for the locally maximum skewness direction  is transformed into the problem of calculating the eigenpairs (the eigenvalues  and the corresponding eigenvectors)  of a coskewness tensor. By combining a fixed-point method with an orthogonal constraint, it can prevent the new eigenpairs from  converging to the same maxima that has been determined before. However,  the eigenvectors of the supersymmetric tensor are  not inherently orthogonal  in general, which implies that  the results obtained by  the  search strategy  used in PSA 
 may unavoidably deviate from  the  actual eigenpairs.   In this paper, we propose a new nonorthogonal search strategy to solve this problem  and  the new algorithm is named  nonorthogonal principal skewness analysis (NPSA). The contribution of NPSA lies in  the finding that the search space   of the eigenvector to be determined can be   enlarged by using the orthogonal complement of   the  Kronecker product of the previous one, instead of  its orthogonal complement space. We give a detailed theoretical proof to illustrate why the new strategy can  result in the more accurate eigenpairs. In addition, after some algebraic derivations, the  complexity of the presented algorithm is  also greatly reduced.   Experiments with both simulated data and real multi/hyperspectral imagery  demonstrate its validity in feature extraction.
 

\end{abstract}


\section{Introduction}
\IEEEPARstart{S}{INCE} hyperspectral  imagery   consists of tens or hundreds of  bands with a very high spectral resolution, it has drawn more attention from various applications in the past  decades, such as spectral unmixing\cite{VICM,NTFunmixing}, classification\cite{CLASS,CKF_SVM}, target detection\cite{RAM,KBTC} and so on. However, high spectral dimensionality with strong intraband correlations also results in informantion redundancy and  computational burden of data processing\cite{GRPCA}.Therefore, dimensionality reduction (DR) has become one of the most important techniques for addressing these problems. DR can be  categorized  into two classes: feature extraction and feature selection.  In this paper, we mainly focus on the former.


The most commonly used feature extraction algorithm is principal component
analysis (PCA)\cite{PCA}, which aims to search for  the  projection direction  that maximizes the variance. Its solution  corresponds to the eigenvectors  of the image's covariance matrix. Several techniques that originated  from PCA have been developed, such as kernel PCA (KPCA)\cite{KPCA} and maximum noise fraction (MNF)\cite{MNF}. KPCA is a nonlinear extension of PCA, which transforms the data into a higher dimensional space via a mapping function and then performs the PCA method. MNF is another popular method for feature extraction, which considers the image quality and selects the signal-to-noise ratio (SNR) as the measure index. 

The methods mentioned above mainly focus on the second-order statistical characteristics of the data. However, the distribution of many real  data sets  usually does not satisfy the Gaussian distribution. Therefore, these methods may have a poor performance and cannot reveal the intrinsical structure of the data. In this case, many methods based on the higher-order statistics have been paid more attention in recent years and have been applied in many remote sensing fields, including anomaly detection\cite{COSD,HOSKPCA}, endmember extraction\cite{HOSEE}, target detection\cite{HOSTD,HOSVD}. Independent component analysis (ICA) is one of the most successful feature extraction techniques. It was  derived from the blind source separation (BSS) application, which attempts to find a linear representation of non-Gaussian data so that the components are statistically
independent, or as independent as possible\cite{ICA}. Several  widely-used  algorithms   include joint approximate diagonalization of eigen-matrices (JADE)\cite{JADE} and Fast   Independent Component Analysis (FastICA)\cite{fastICA}. JADE utilizes the fourth-order culumant tensor of the data. Some algorithms that originated from JADE  have been developed later.   A third-order analogue called Skewness-based ICA via Eigenvectors of Cumulant Operator (EcoICA) is  proposed\cite{EcoICA}. Other joint diagonalization methods include subspace fitting method (SSF)\cite{SSF}, Distinctively
Normalized Joint Diagonalization (DNJD)\cite{DNJD} and Alternating Columns-Diagonal Centers (ACDC)\cite{NJD}, etc.

FastICA can select skewness,  negentropy or other indices as the  non-Gaussian measurement. It can reach a cubic convergence rate and outperform most of the other commonly used ICA algorithms\cite{Oja}. However,  it requires all the pixels to be involved in each iteration for searching for  the optimal projection direction, which is quite time-consuming, especially for the high dimensional data. To solve this problem, Geng et al. has proposed an efficient method called Principal Skewness Analysis (PSA)\cite{PSA}. PSA can be viewed as a third-order generalization of PCA. Meanwhile, it is also equivalent to FastICA when selecting the skewness as a non-Gaussian index. Following this work, a momentum version (called MPSA) to alliviate  the oscillation phenomenon of PSA \cite{MPSA} and a natural fourth-order extension method, i.e., Principal Kurtosis Analysis (PKA)\cite{PKA} are also analyzed.

 The solution  of these PSA-derived methods can  be transformed into the problem of calculating the eigenvalues and the corresponding eigenvectors of the tensor, which is similar to PCA. By adopting the fixed-point scheme, the solution can be obtained iteratively.   To  prevent the solution from converging to the same one that has been determined previously, an orthogonal complement operator is also introduced in these methods. Thus  all of them   can obtain an orthogonal  transformation matrix eventually,  since the search space of the eigenvector to be determined is always restricted  in the orthogonal complement space of the previous one. However, theoretical analysis based on multi-linear algebra has shown that the eigenvectors of a supersymmetric tensor are not inherently orthogonal in general \cite{NCM,Hsu}, which is different from the situation for the real symmetric matrix.  Thus the orthogonal constraint in PSA and the inherent nonorthogonality of the eigenvectors of  supersymmetric tensor are a pair of  irreconcilable  contradictions.  In this paper, we propose  a more relaxed  constraint, and based on this,   a new  algorithm, which is  named nonorthogonal principal skewness analysis (NPSA), is presented  to deal with this problem.   It is expected  that  NPSA can have the following two attributes: 1) similar to PSA,  it can also prevent the eigenvector to be determined from converging to the  eigenvectors that have been determined previously; 2)   it can obtain a more accurate solution than that of PSA meanwhile.

The rest of this paper is organized as follows. In
Section II, we  briefly review the original  PSA algorithm and analyze its deficiencies with a simple example. In Section III, we present  the new strategy in NPSA first and then obtain an  improved version  when taking  the complexity into
consideration. Two strategies are compared in the end of this section. In Section IV, we give  some theoretical analysis to justify  the validity of  the algorithm.  Some experimental results are given in Section V, and the conclusions are drawn in Section VI.
\section{Background}
   In this section, we first introduce some notations, basic definitions  and important properties used throughout this paper, and then  give a brief review of  the formations and deficiencies  of PSA.

 \begin{figure}[t]
    \centering
  \includegraphics[width=0.45\textwidth]{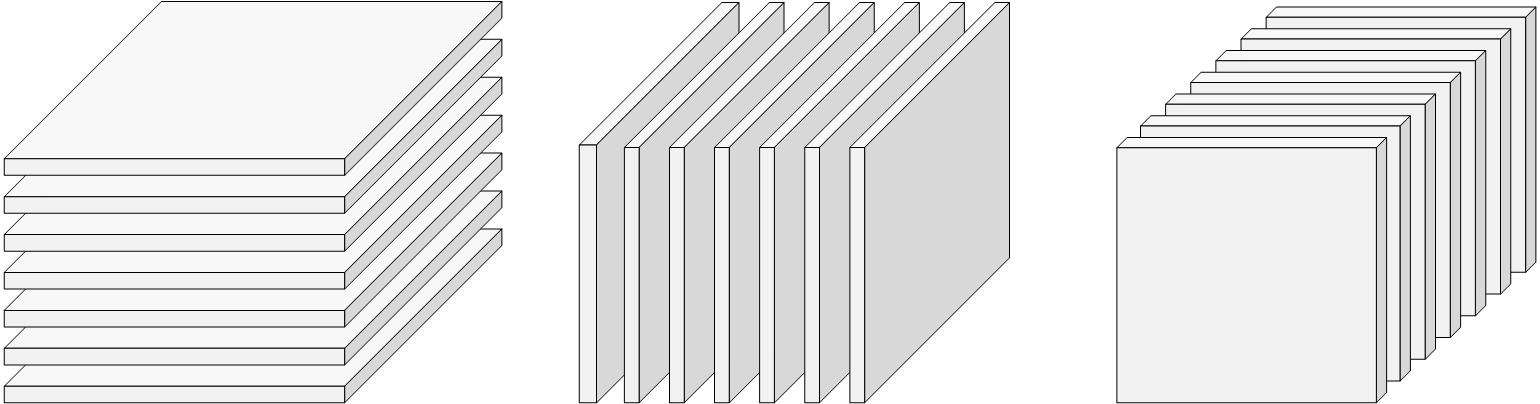}\\
  \caption{\quad Slices of a third-order tensor, i.e., horizontal slices $ \mathbf A_{i::}$, lateral slices $\mathbf A_{:j:} $ and  frontal slices $\mathbf A_{::k}$.}
  \label{sliceshow}
\end{figure}

\subsection{Preliminaries}
   Following \cite{kolda}, in this paper,  the high-order tensors are denoted by boldface Euler script letters, e.g., $\mathcal  A $. A $N$th-order tensor is defined as $\mathcal A \in R^{I_1 \times I_2  \times \dots \times I_N } $, where $N$ is the order of $\mathcal A $, also called the  way or   mode. For $N$=$1$, it is a vector.  For $N$=$2$, it is a matrix. The element of $\mathcal A$ is denoted by $a_{i_1,i_2,\dots,i_N},i_n \in \{{1,2,\dots,I_n}\},1 \le n \le N$.
   Fibers, the higher-order analogue of matrix rows and columns, are defined by fixing every index except for  one. Slices are two-dimensional sections of a tensor, defined by fixing all except for two indices. For a third-order tensor $\mathcal A \in R^{I_1 \times I_2 \times I_3}$, as shown in Fig.~\ref{sliceshow}, its three different slices are called horizontal, lateral and frontal slices, which can be denoted by $ \mathbf A_{i::},\mathbf A_{:j:},\mathbf A_{::k}$, respectively. Compactly, the $ k $th frontal slice  is also denoted as $ \mathbf A_k$. A tensor is called supersymmetric if its elements remain invariant under any permutation of the indices\cite{kolda}.

   Some important operations are illustrated as follows: $\circ $ denotes the outer product of two vectors. The operator $\rm vec $  is to reorder the elements of a matrix or a higher-order tensor into a vector and $\rm unvec $  is the opposite. The $n$-mode product of a tensor $\mathcal A \in R^{I_1 \times I_2 \dots \times I_N } $ with a matrix $\mathbf U \in R^{J \times I_n}$ is   denoted by $\mathcal A \times_{n} \mathbf U \in R^{I_1 \times \dots \times I_{n-1} \times J \times I_{n+1}\times \dots \times I_N } $, whose element is $(\mathcal A \times_{n} \mathbf U )_{i_1\dots i_{n-1}ji_{n+1}\dots i_N}=\sum_{i_{n}=1} ^{I_{n}} a_{i_1,i_2,\dots,i_N} u_{ji_{n}} $. The range of a matrix $ \mathbf A
    \in R^{m \times n}$ is defined as   $ \rm R( \mathbf A) =\{  \mathbf y \in \it{R}^{m } | \mathbf A\mathbf x=\mathbf y, \mathbf x \in \it{R}^{n}  \}$ and its dimensionality is denoted by $ \rm dim[\rm R( \mathbf A)] $. $N^{+}$ denotes the set of all positive integers.

    The Kronecker product of two matrices $ \mathbf A \in R^{I \times J} $ and $ \mathbf B \in R^{K \times L} $ is a matrix denoted as $
  \mathbf A  \otimes  \mathbf B  \in R^{IK \times JL} $, which is defined as
  \begin{equation}
    \begin{split}
  \mathbf A  \otimes  \mathbf B
  &=\begin{bmatrix}
   a_{11} \mathbf B  & \dots &  a_{1J}  \mathbf B  \\
   & \ddots  & \vdots  \\
  a_{I1}  \mathbf B  & \dots  & a_{IJ}  \mathbf B
   \end{bmatrix}
   \end{split}.
\end{equation}

 For simplicity,  we use $ \mathbf A ^{\otimes ^{p}}$ and $ \mathbf a ^{\otimes ^{p}}$ to denote the $p$-times  Kronecker Product of the matrix $ \mathbf A  $ and the vector $ \mathbf a $,  respectively.

The properties that will be used later are presented here\cite{KP,zhang2017matrix},
\begin{equation}\label{kp1}
(\mathbf A \otimes \mathbf B)^{\mathrm {T}}=\mathbf A^{\mathrm {T}} \otimes  \mathbf B^{\mathrm {T}},
\end{equation}
\begin{equation}\label{kp3}
(\mathbf A \otimes \mathbf B)^{\mathrm {-1}}=\mathbf A^{\mathrm {-1}} \otimes  \mathbf B^{\mathrm {-1}},
\end{equation}
\begin{equation}\label{kp2}
(\mathbf A \otimes \mathbf C)(\mathbf B \otimes \mathbf D)=(\mathbf A\mathbf B)\otimes(\mathbf C\mathbf D),
\end{equation}
\begin{equation}\label{KPrank}
  \rm rank(\mathbf A \otimes  \mathbf B)=\rm rank(\mathbf A)\rm rank(\mathbf B),
  \end{equation}
\begin{equation}\label{vec1}
 \rm vec(\mathbf A +\mathbf B ) =\rm vec(\mathbf A)+\rm vec(\mathbf B),
\end{equation}
\begin{equation}\label{vec2}
 \rm vec(\mathbf A \mathbf B \mathbf C) =(\mathbf C^{\mathrm T} \otimes \mathbf A) \rm vec(\mathbf B ),
\end{equation}
\begin{equation} \label{tensorkron}
 \begin{split}
 \shoveleft \rm vec( \mathcal S \times_1 \mathbf A^{(1)}  \times_2 \mathbf A^{(2)} \dots  \times_n \mathbf A^{(n)}  ) =  \\
    (\mathbf A^{(n)}   \otimes \mathbf A^{(n-1)} \otimes \dots \otimes \mathbf A^{(1)} )^{\mathrm T} \rm vec( \mathcal S ).
\end{split}
\end{equation}


 \begin{figure}[t]
    \centering
  \includegraphics[width=0.4\textwidth]{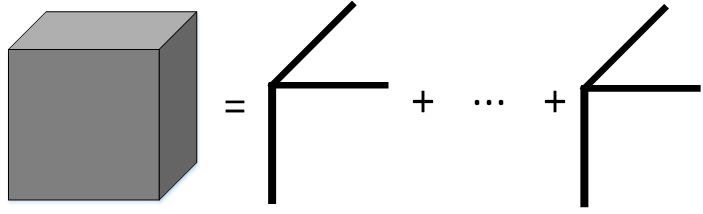}\\
  \caption{\quad Sktech map of the calculation of the coskewness tensor $\mathcal S$. }
  \label{sketch}
\end{figure}

%
%
%

\subsection{PSA Algortithm}
In PSA, the coskewness tensor, the analogue of the covariance matrix in PCA, is constructed to calculate the skewness of the image in the direction $ \mathbf u$.

Assuming that the image data set is $\mathbf X=[\mathbf x_1,\mathbf x_2,\dots, \mathbf x_{N}] \in R^{L \times N}$, where $\mathbf x_i $ is the $ L \times 1 $ vector and $ N $ is the number of pixels.  The image should  first be centralized and whitened by
\begin{equation}\label{whiten}
\mathbf R=\mathbf {F}^{\mathrm {T}}(\mathbf X-\mathbf m),
\end{equation}
where $ \mathbf m =(1/N) \sum_{i=1}^{N} \mathbf x_i $ is the mean vector and $\mathbf R=[\mathbf r_1,\mathbf r_2, \dots , \mathbf r_N] $ is the whitened image. $ \mathbf F=\mathbf E\mathbf D^{-\frac {1} {2}  }  $ is called the whitening operator, where   $\mathbf E $ represents the eigenvector matrix of the covariance matrix and $ \mathbf D $ is the corresponding eigenvalue diagonal matrix.

Then, the coskewness tensor is calculated by
\begin{equation} \label{tensorcompute}
\mathcal S =\frac 1 N \sum \limits_{i=1}^{N} \mathbf r_i \circ \mathbf r_i \circ \mathbf r_i.
\end{equation}

 Fig.~\ref{sketch}  shows a sketch map of the calculation of $\mathcal S$.  Obviously, $\mathcal S  $ is a supersymmetric  tensor with a size  of $ L \times L \times L $. 

 Then the skewness of an image in any direction $ \mathbf u $ can be calculated by
\begin{equation}\label{tensorskew}
\rm skew(\mathbf u )=\mathcal S  \times_{1} \mathbf u  \times_{2} \mathbf u  \times_{3} \mathbf u,
\end{equation}
where  $ \mathbf u \in R^{L \times 1} $ is a unit vector, i.e., $\mathbf u^{\mathrm {T}}\mathbf u=1$. 

So the optimization model  can be formulated as 
\begin{equation} \label{model}
\begin{cases}
\max\limits_{\mathbf u}  \mathcal S  \times_{1} \mathbf u  \times_{2} \mathbf u  \times_{3} \mathbf u   \\
s.t. \mathbf u^{\mathrm {T}}\mathbf u=1
\end{cases}.
\end{equation}
Using the Lagrangian method, the problem is equivalent to solving the equation
\begin{equation} \label{solution}
\mathcal S  \times_{1} \mathbf u   \times_{3}\mathbf u  =\lambda \mathbf u.
\end{equation}

 A fixed-point method is performed to calculate  $\mathbf u$ for each unit, which can be expressed as follows:
 \begin{equation} \label{iteration}
\begin{cases}
\mathbf u=\mathcal S  \times_{1} \mathbf u   \times_{3}\mathbf u   \\
\mathbf u=\mathbf u/ \Vert \mathbf u \Vert_2
\end{cases}.
\end{equation}

If it does have a fixed-point, the solution $\mathbf u $ is called the first principal skewness direction and $ \lambda $ is the skewness of the image in the direction $ \mathbf u $. Equivalently, $(\lambda ,\mathbf u )$ is also called the eigenvalue/eigenvector pair of a tensor, introduced by Lim \cite{lim} and Qi\cite{qi}.

To prevent the second eigenvector from converging to the same one as the first, the algorithm projects the data into the orthogonal complement space of $\mathbf u$, which is equivalent to generate a new tensor by  calculating the $n$-mode product
 \begin{equation}\label{oldupdate}
 \mathcal S=\mathcal S  \times_{1} \mathbf P_{\mathbf u }^{\bot} \times_{2} \mathbf P_{\mathbf u }^{\bot} \times_{3} \mathbf P_{\mathbf u }^{\bot},
\end{equation}
where $ \mathbf P_{\mathbf u }^{\bot} =\mathbf I- \mathbf u (\mathbf u^{\mathrm {T}} \mathbf u)^{-1} \mathbf u^{\mathrm {T}} $ is the orthogonal complement projection operator of $ \mathbf u $ and $ \mathbf I $ is the $ L \times L $ identity matrix.

Then,   the same iteration method, i.e., (\ref{iteration}), can be applied to  the new tensor  $ \mathcal S $ to obtain the second eigenvector and the following  process is conducted in the same manner.


\subsection{Deficiencies  Of PSA}
As  mentioned before, an orthogonal complement operator is introduced in PSA in order to prevent the next eigenvector from  converging to the  eigenvectors that have been determined previously. As is well known, the eigenvectors of a real symmetric matrix is naturally orthogonal to each other. However, this may not hold when generalized to the higher-order cases.

We here present a simple example to  illustrate this phenomenon. Consider a supersymmetric tensor $ \mathcal S \in R^{2 \times 2 \times 2}$, whose two frontal  slices are
\begin{equation}
   \begin{split}
  \notag
  \mathbf S_{1}=\begin{bmatrix}
   2  & -1   \\
   -1  &  0.8
   \end{bmatrix},
      \mathbf S_{2}=\begin{bmatrix}
   -1  & 0.8  \\
   0.8  &  0.3
   \end{bmatrix}
   \end{split}.
\end{equation}

 It can easily verifed that its two eigenvectors are
\begin{equation}
\notag
  \mathbf u_1=[0.8812,-0.4727]^{\mathrm T},\mathbf u_2=[0.3757 , 0.9267]^{\mathrm T},
\end{equation}
and their inner product is $   \mathbf u_1^{\mathrm T}\mathbf u_2 =-0.1070$, which means that they are nonorthogonal. However,  the results obtained by  the PSA algorithm are
 \begin{equation}
\notag
  \mathbf u_1^{\rm PSA}=[0.8812,-0.4727]^{\mathrm T},\mathbf u_2^{\rm PSA}=[0.4727, 0.8812]^{\mathrm T},
\end{equation}
 which are orthogonal. It is apparent that $\mathbf u_2^{\rm PSA}$ deviates  from $\mathbf u_2$, and  they have a $6.1430^{\circ} $  angle.  The error  is caused by  the orthogonal constraint  used  in PSA.  Therefore, how to  obtain the more accurate  eigenpairs is  significant.

\section{NPSA}


\subsection{New Search Strategy}

Here, we first give the new search strategy in NPSA and then theoretically illustrate why this method can obtain the more accurate eigenpairs  in the next section.

Similar to PSA, the first eigenvector $ \mathbf u $ can be  obtained according to (\ref{iteration}). The subsequent steps are presented as follows:

\textbf{(1)}: vectorize the tensor $\mathcal S $  into a vector $ \mathbf s$. Usually, The vectorization of a third-order tensor is  defined as the vertical arrangement of column vectorization of the front-slice matrix \cite{zhang2017matrix}, i.e., $\mathbf s= \rm vec(\mathcal S )=[\mathbf {s}_1^{\mathrm {T}},\mathbf {s}_2^{\mathrm {T}},\dots,  \mathbf {s}_{\it L}^{\mathrm {T}} ]^{\mathrm {T}}$ $\in$  $R^{ L^{3} \times 1}$, where  $\mathbf {s}_i \in R^{L^{2} \times 1} $  is the vector generated by the $i$-th frontal slice $ \mathbf S_i $ , i.e., $\mathbf {s}_i=\rm vec(\mathbf S_i )$ . 

\textbf{(2)}: compute a new vector via the $3$-times Kronecker product of the vector $ \mathbf u $, denoted by $ \mathbf u ^{\otimes ^{3}}=\mathbf u \otimes  \mathbf u  \otimes  \mathbf u$ .

\textbf{(3)}:  compute the  orthogonal complement  projection matrix of $ \mathbf u ^{\otimes ^{3}}$, which can be expressed as
\begin{equation}\label{step3}
 \mathbf {P}^{\bot}_{\mathbf u ^{\otimes ^{3}} }   =   \mathbf I^{\otimes ^{3}} -                \mathbf u ^{\otimes ^{3}}        [  (\mathbf u ^{\otimes ^{3}} )^{\mathrm {T}} \mathbf u ^{\otimes ^{3}}  ]^{-1} (\mathbf u ^{\otimes ^{3}} )^{\mathrm {T}},
 \end{equation}
 where $ \mathbf I^{\otimes ^{3}} $ is the 3-times Kronecker product of the matrix $\mathbf I $ of  (\ref{oldupdate}) and  it is a $ L^{3} \times L^{3} $ identity matrix.

\textbf{(4)}: multiply $ \mathbf {P}^{\bot}_{\mathbf u ^{\otimes ^{3}} } $ and the vectorized tensor $ \mathbf s$ in step (1) and then  perform the $\rm unvec $ operation to obtain  a new tensor. Without causing ambiguity and to  have a concise form, we still express the updated tensor as  $\mathcal S $. Thus we can have
\begin{equation}\label{step4}
\mathcal S= \rm unvec (\mathbf {P}^{\bot}_{\mathbf u ^{\otimes ^{3}} } \cdot \mathbf s ).
\end{equation}

Then we can obtain the second eigenvector by performing the fixed-point scheme, i.e.,   (\ref{iteration}),  to the new updated tensor and similarly repeat the above process until we get all or pre-set eigenpairs.

For the example mentioned in the previous section,  the two eigenvectors obtained by NPSA   are
 \begin{equation}
\notag
  \mathbf u_1^{\rm NPSA}=[0.8812,-0.4727]^{\mathrm T},\mathbf u_2^{\rm NPSA}=[0.3351, 0.9422]^{\mathrm T}.
\end{equation}

  The angle between $\mathbf u_2^{\rm NPSA} $ and $\mathbf u_2 $ is $2.4874^{\circ}$, which is  less than that between $\mathbf u_2^{\rm PSA} $ and $\mathbf u_2 $, as shown in Fig.~\ref{eigenvectors}. It means that   the  new search strategy presented  in NPSA is actually efficient.


\begin{figure}
  \centering
\begin{tikzpicture}[scale=2]
\draw[thick] (0,0) circle(1);
\draw[->] (-1.2,0)  -- (1.2,0);
\draw[->]  (0,-1.2)  -- (0,1.2);
\fill  (0:0)  circle(0.5pt);

\draw (-1, -0.02)  -- (-1, 0.02)  ;
\node[below,outer sep=2pt, fill=white,font=\scriptsize] (P)  at (180:1.15) {$ -1$} ;

\draw (-0.5, -0.02)  -- (-0.5, 0.02)  ;
\node[below,outer sep=2pt, fill=white, font=\scriptsize] (P)  at (180:0.5) {$ -0.5$ } ;

\draw (0.5, -0.02)  -- (0.5, 0.02);
\node[below,outer sep=2pt, fill=white, font=\scriptsize] (P)  at (0:0.5) {$ 0.5$ } ;

\draw (1, -0.02)  -- (1, 0.02)  ;
\node[below,outer sep=2pt, fill=white, font=\scriptsize] (P)  at (0:1.1) {$ 1$ } ;

\draw ( -0.02, 1)  -- ( 0.02, 1)  ;
\node[left,outer sep=2pt, fill=white, font=\scriptsize] (P)  at (90:1.1) {$ 1$ } ;

\draw ( -0.02, 0.5)  -- ( 0.02, 0.5)  ;
\node[left,outer sep=2pt, fill=white, font=\scriptsize] (P)  at (90:0.5) {$ 0.5$ } ;

\node[above,left,outer sep=2pt, fill=white, font=\scriptsize] (P) at (90:0.13) {$ 0 $};
\draw ( -0.02, -0.5)  -- ( 0.02, -0.5)  ;
\node[left,outer sep=2pt, fill=white, font=\scriptsize] (P)  at (270:0.5) {$ -0.5$ } ;

\draw ( -0.02, -1)  -- ( 0.02, -1)  ;
\node[left,outer sep=2pt, fill=white, font=\scriptsize] (P)  at (270:1.15) {$ -1$ } ;

\draw[->] (0,0) -- (0.8812,-0.4727);
\fill  (0.8812,-0.4727)  circle(0.5pt);
\draw[->,color=red] (0,0) -- (0.3351,0.9422);
\fill[color=red]  (0.3351,0.9422)  circle(0.5pt);
\draw[->,color=blue] (0,0) -- (0.4727,0.8812);
\fill[color=blue]  (0.4727,0.8812)  circle(0.5pt);
\draw[->] (0,0) -- (0.3757,0.9267);
\fill  (0.3757,0.9267)  circle(0.5pt);

\node (P)  at (325:1.1) {$ \mathbf u_1 =\mathbf u_1^{\rm PSA}=\mathbf u_1^{\rm NPSA} $ } ;
\node (P)  at (50:1.2) {\color{blue} $ \mathbf u_2^{\rm PSA}$ } ;
\node (P)  at (75:1.2) {  \color{red}  $ \mathbf u_2^{ \rm NPSA}$ } ;
\node (P)  at (65:1.1) {$ \mathbf u_2$ } ;
  \end{tikzpicture}
  \caption{ The  distribution  of the true eigenvectors, and those obtained by NPSA and PSA  in an unit circle.}
  \label{eigenvectors}
\end{figure}
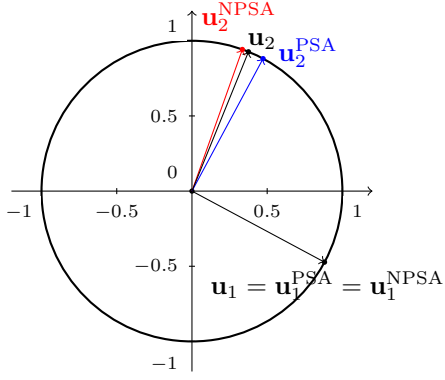




\subsection{ Complexity Reduction}
 In this subsection, we take   the  complexity  into
consideration.  It can be observed that although the  strategy shown in step (1) $\sim $ step (4) is efficient, there are still some problems in the implementation: 1) it needs to perform the $\rm vec $ and $\rm unvec $ operations repeatedly; 2)  computing the orthogonal projection matrix $ \mathbf {P}^{\bot}_{\mathbf u ^{\otimes ^{3}} }  $  takes up an  $ L^{3} \times L^{3} $ memory. When $ L $ becomes larger, especially for  hyperspectral images with tens or hundreds of bands, the computational burden is huge and unbearable.  So in the following, we try to  reduce the  computational complexity and  to save the  storage memory  simultaneously.

For  (\ref{step3}), based on (\ref{kp1}) and (\ref{kp2}),
we can derive
\begin{equation}\label{simple1}
\begin{split}
(\mathbf u ^{\otimes ^{3}} )^{\mathrm {T}} \mathbf u ^{\otimes ^{3}}
&=(\mathbf u \otimes  \mathbf u  \otimes  \mathbf u)^{\mathrm {T}}  ( \mathbf u \otimes  \mathbf u  \otimes  \mathbf u   )  \\
&= (\mathbf u^{\mathrm {T}}\mathbf u)\otimes  (\mathbf u^{\mathrm {T}}\mathbf u)  \otimes  (\mathbf u^{\mathrm {T}}\mathbf u)   \\
&=1.
 \end{split}
\end{equation}

It means that  the vector generated by the $3$-times Kronecker product of  a unit vector  is still with a unit length. In this way,  (\ref{step3}) can be simplified as
\begin{equation} \label{simple2}
\begin{split}
\mathbf {P}^{\bot}_{\mathbf u ^{\otimes ^{3}} }
&=\mathbf I^{\otimes ^{3}}- \mathbf u ^{\otimes ^{3}} (\mathbf u ^{\otimes ^{3}} )^{\mathrm {T}}    \\
&=\mathbf I^{\otimes ^{3}}-  (\mathbf u\mathbf u^{\mathrm {T}})\otimes  (\mathbf u\mathbf u^{\mathrm {T}})  \otimes  (\mathbf u\mathbf u^{\mathrm {T}}).
\end{split}
\end{equation}


According to (\ref{step4}), the new tensor  $\mathcal S $ can be  updated by
\begin{equation}\label{etensor}
\begin{split}
\mathcal S
&= \rm unvec (\mathbf {P}^{\bot}_{\mathbf u ^{\otimes ^{3}} } \cdot \mathbf s ) \\
&=\rm unvec ([\mathbf I^{\otimes ^{3}}-  (\mathbf u\mathbf u^{\mathrm {T}})\otimes  (\mathbf u\mathbf u^{\mathrm {T}})  \otimes  (\mathbf u\mathbf u^{\mathrm {T}})] \cdot \mathbf s   )        \\
&=\mathcal S  - \rm unvec[((\mathbf u \mathbf  {u}^{\mathrm {T}})\otimes (\mathbf u \mathbf  {u}^{\mathrm {T}}) \otimes (\mathbf u \mathbf  {u}^{\mathrm {T}}  ))\cdot \mathbf s]    \\
&=\mathcal S  -\mathcal {\tilde S },
\end{split}
\end{equation}
where we  introduce an auxiliary tensor, denoted by
\begin{equation}
\mathcal {\tilde S }  =\rm unvec[((\mathbf u \mathbf  {u}^{\mathrm {T}})\otimes (\mathbf u \mathbf  {u}^{\mathrm {T}}) \otimes (\mathbf u \mathbf  {u}^{\mathrm {T}}  ))\cdot \mathbf s].
\end{equation}

For simplicity, let

\begin{equation}
 \mathbf A=[(\mathbf u \mathbf  {u}^{\mathrm {T}}  )\otimes (\mathbf u \mathbf  {u}^{\mathrm {T}}  )]_{L^{2} \times L^{2}},
\end{equation}
then we can have
\begin{equation}\label{ekronecker}
\begin{split}
  (\mathbf u \mathbf  {u}^{\mathrm {T}} )\otimes (\mathbf u \mathbf  {u}^{\mathrm {T}}  ) \otimes (\mathbf u \mathbf  {u}^{\mathrm {T}} )
  &=\begin{bmatrix}
   u_1 u_1 \mathbf A  & \dots &  u_1 u_L \mathbf A  \\
   & \ddots  & \vdots  \\
   u_L u_1 \mathbf A  & \dots  & u_L u_L \mathbf A
   \end{bmatrix}.
   \end{split}
\end{equation}

Since
\begin{equation}
\mathbf s= \rm vec(\mathcal S )=[\mathbf {s}_1^{T},\mathbf {s}_2^{T},\dots,  \mathbf {s}_{\it L}^{T} ]^{T},
\end{equation}
   we denote
 \begin{equation}\label{ess}
 \mathbf {\tilde s }= \rm vec (  \mathcal {\tilde S }  )=[\mathbf {\tilde s }_1^{T},\mathbf { \tilde s }_2^{T}, \dots, \mathbf { \tilde s}_{\it L}^{T}]^{T},
 \end{equation}
 where $\mathbf {\tilde s}_i $  is the vector generated by the $i$-th frontal slice $ \mathbf {\tilde S}_i $, i.e.,  $\mathbf {\tilde s}_i=\rm vec(\mathbf {\tilde S}_i )$.


Then, according to (\ref{ekronecker}) $\sim$  (\ref{ess}), we can derive
 \begin{equation}
  \mathbf  {\tilde {s}} _j =\sum \limits _{i=1}^{L} u_j u_i   \mathbf A  \cdot \mathbf {s} _i = u_j \sum \limits _{i=1}^{L}   u_i \mathbf A \cdot \mathbf {s} _i.
 \end{equation}

 The $j$-th slice of the auxiliary tensor  can be expressed as
 \begin{equation}\label{simpleversion}
\begin{split}
\tilde {{\mathcal S}}_j
 &=\rm unvec(\mathbf {\tilde {s} }_j )   \\
 &= u_j \sum \limits _{i=1}^{L}  u_i \rm {unvec}( \mathbf {A} \cdot \mathbf {s} _i)     \\
 &= u_j \sum \limits _{i=1}^{L} u_i \rm unvec \{ [(\mathbf u \mathbf  {u}^{\mathrm T})\otimes (\mathbf u \mathbf  {u}^{\mathrm T})]  \cdot \mathbf s_i   \}  \\
 &= u_j \sum \limits _{i=1}^{L}  u_i (\mathbf u \mathbf  {u}^{\mathrm T} {\mathbf S}_i \mathbf u \mathbf  {u}^{\mathrm  T}  ),
\end{split}
\end{equation}
where  (\ref{vec1}) and  (\ref{vec2}) are utilized.

Recalling the definition of the $n$-mode product, (\ref{simpleversion}) can be equivalent  to
\begin{equation}
 \tilde {\mathcal S}=\mathcal S \times_1 (\mathbf u \mathbf  {u}^{\mathrm T} )  \times_2 (\mathbf u \mathbf  {u}^{\mathrm T} )  \times_3 (\mathbf u \mathbf  {u}^{\mathrm T} ),
\end{equation}
so the new updated tensor can be expressed as
\begin{equation} \label{updatedtensor}
 \mathcal S=\mathcal S-\mathcal S \times_1 (\mathbf u \mathbf  {u}^{\mathrm T} )  \times_2 (\mathbf u \mathbf  {u}^{\mathrm T} )  \times_3 (\mathbf u \mathbf  {u}^{\mathrm T} ).
\end{equation}

Thus, we obtain a more compact representation for  the  tensor update. We name it  the improved strategy, as opposed to the originally proposed one described in  step (1) $\sim $ step (4). It should be noted that the subtraction operation in (\ref{updatedtensor}) corresponds to the orthogonal  complement projection operation in (\ref{step3}).

Interestingly, we can compare (\ref{updatedtensor}) with the update formula of PSA defined in  (\ref{oldupdate}), which we  can restate here
\begin{equation}\label{oldmethod}
 \mathcal S=\mathcal S \times_1 (\mathbf I-\mathbf u \mathbf  {u}^{\mathrm T} )  \times_2 (\mathbf I-\mathbf u \mathbf  {u}^{\mathrm T} )  \times_3 (\mathbf I-\mathbf u \mathbf  {u}^{\mathrm T} ),
\end{equation}
since $  \mathbf  {u}^{\mathrm T}\mathbf u =1$.


In a sense, two strategies shown in   (\ref{updatedtensor})   and  (\ref{oldmethod}) differ in the order in which they perform the orthogonal  complement projection and the $n$-mode product operation. PSA  generates the orthogonal complement projection matrix first and then calculate the $n$-mode product to update a new  tensor. In contrast,  NPSA  first  obtains  an  auxiliary tensor via the $n$-mode product, followed by the orthogonal complement projection  operation.



\subsection{Complexity Comparison }
Here, we give a detailed comparison for the two different strategies  from two aspects, including the  required maximum storage memory and the computational  complexity. 

On the one hand,  the original  strategy needs to calculate a  large-scale orthogonal projection matrix of size $ L^{3} \times L^{3} $ and rearrange the elements repeatedly, while the improved  version only takes up  $ L \times L \times L$ memory to store the auxiliary tensor, which can  greatly save the memory.

On the other hand,  the  computational complexity of both step (3) and step (4) is $O(L^6)$, which is very time-consuming, especially when $L$ is large. In contrast, the improved version can have a lower computational complexity. It can be checked that   the computational complexity to update the auxiliary tensor  in (\ref{updatedtensor}) is   $ O(L^3) $.  Table~\ref{cc1} concludes the  complexity  comparison of the two strategies.

\begin{table}[t]
    \normalsize
    \centering
    \caption{\\ Comparison of the two strategies with respect to maximum storage memory and computational complexity.  }
    \label{cc1}
\renewcommand\arraystretch{1.3}
    \begin{tabular}{  c  c  c    }
		\hline
		\cline{1-3}
		  &  original  &  improved \\
		\hline
        \cline{1-3}	
          maximum storage memory  & $L^3 \times L^3 $ & $L \times L \times L $   \\
          computational complexity  &  $O(L^6)$ &  $O(L^3) $\\
		\hline
	\end{tabular}
\end{table}

\begin{algorithm}[t]
\caption{Nonorthogonal PSA (NPSA)}
\label{Algo.1}
\begin{algorithmic}[1]
\REQUIRE
 image data $\mathbf X  $, and the number of the principal skewness
directions, $ p $.
\ENSURE
the transformation matrix $\mathbf U $  and the corresponding transformed image $\mathbf Y= \mathbf U^{\mathrm T} \mathbf R  $.

\% \textbf{ main precedure }:\%
\STATE
centralize a nd whiten the data according to (\ref{whiten}).
\STATE
calculate  the coskewness tensor $\mathcal S $. 

\% \textbf{ main loop  }:\%
\FOR{$ i=1:p $}
\STATE
 let  $k=0$.
\STATE
initialize the  $ \mathbf u_i^{(k)} $ with random unit vector.
\WHILE{stop conditions are not met}  \label{whileloop}
\STATE
$\mathbf u_i^{(k+1)} =\mathcal S  \times_{1} \mathbf u_i^{(k)}   \times_{3} \mathbf u_i^{(k)}$,
\STATE
$\mathbf u_i^{(k+1)} =\mathbf u_i^{(k+1)} /\Vert \mathbf u_i^{(k+1)} \Vert_2  $,
\STATE
$k=k+1$,
\ENDWHILE
\STATE
$\mathbf U_{:i}=\mathbf u_i^{(k+1)}$,

\STATE
$\mathbf u=\mathbf u_i^{(k+1)}$,
\STATE \label{aux}
 update the tensor according  to  (\ref{updatedtensor}).
\ENDFOR
\end{algorithmic}
\end{algorithm}

Finally, the pseudo-code of NPSA is summarized in Algorithm~\ref{Algo.1}.

Some comments are described as follows. Generally, the stop conditions in step (\ref{whileloop}) include error tolerance $\epsilon $  and maximum times  $K$.  In this paper,  $\epsilon$ is set to  0.0001, and  $K$ is set to 50.    $\mathbf U \in R^{p \times p} $ is the  final nonorthogonal principal skewness transformation matrix.



\section{Theoretical Analysis}


In the above section, we have demonstrated that NPSA outperforms PSA using a simple example. Now,
 to  theoretically illustrate why the  former can obtain the more accurate solutions, we present  the following lemma.
 \\

 \begin{lemma} \label{lemma1}
 Consider an $n \times m $  column-full rank matrix $ \mathbf S $, it  holds that the orthogonal complement  of the space spanned by the  Kronecker product of $ \mathbf S $  always contains the space spanned by the  Kronecker product of its orthogonal complement operator, which can be expressed as follows
 \begin{equation} \label{lemmacontent}
  \rm R(  (\mathbf {P}^{\bot}_{\mathbf S} )^{\otimes ^{\it p}}    )  \subseteq  \rm R(  \mathbf {P}^{\bot}_{\mathbf S ^{\otimes ^{\it p}} }  )
  \end{equation}
  for  $ \forall $  $ p \in N^{+} $.
  \end{lemma}

\begin{proof}
We start by defining
\begin{equation}\label{zjbjianhua}
  \mathbf {P}^{\bot}_{\mathbf S}=\mathbf I-\mathbf S (\mathbf S^{\mathrm T} \mathbf S)^{-1} \mathbf S^{\mathrm T}.
  \end{equation}

 Denote $ \mathbf Q =\mathbf S (\mathbf S^{\mathrm T} \mathbf S)^{-1} \mathbf S^{\mathrm T} $, and we have

  \begin{equation}\label{leftsimplicity}
  (\mathbf {P}^{\bot}_{\mathbf S} )^{\otimes ^{p}}  =(  \mathbf I-\mathbf Q ) ^{\otimes ^{p}}.
  \end{equation}

  Similarly,
  \begin{equation}\label{rightsimplicity1}
  \mathbf {P}^{\bot}_{\mathbf S ^{\otimes ^{p}} }=\mathbf I ^{\otimes ^{p}} - \mathbf S ^{\otimes ^{p}}        [  (\mathbf S ^{\otimes ^{p}} )^{\mathrm T} \mathbf S ^{\otimes ^{p}}  ]^{-1} (\mathbf S ^{\otimes ^{p}} )^{\mathrm T}.
  \end{equation}

  According to the property (\ref{kp1}) $\sim$  (\ref{kp2}), it can be derived that
\begin{equation}\label{rightsimplicity2}
\begin{split}
\mathbf {P}^{\bot}_{\mathbf S ^{\otimes ^{p}} }
&=\mathbf I ^{\otimes ^{p}} - \mathbf S ^{\otimes ^{p}}   [  (\mathbf S^{\mathrm T})^{\otimes ^{p}} \mathbf S ^{\otimes ^{p}}  ]^{-1} (\mathbf S^{\mathrm T})^{\otimes ^{p}}                \\
&=\mathbf I ^{\otimes ^{p}}  -[\mathbf S (\mathbf S^{\mathrm T} \mathbf S)^{-1} \mathbf S^{\mathrm T} ]^{\otimes ^{p}}   \\
  &=\mathbf I ^{\otimes ^{p}}  -\mathbf Q ^{\otimes ^{p}}.
  \end{split}
  \end{equation}

  Then
  \begin{equation}\label{leftright}
  \begin{split}
    \mathbf {P}^{\bot}_{\mathbf S ^{\otimes ^{p}} } (\mathbf {P}^{\bot}_{\mathbf S} )^{\otimes ^{p}}
    &=(\mathbf I ^{\otimes ^{p}}  -\mathbf Q ^{\otimes ^{p}} ) (  \mathbf I-\mathbf Q ) ^{\otimes ^{p}}           \\
    &=\mathbf I ^{\otimes ^{p}}(  \mathbf I-\mathbf Q ) ^{\otimes ^{p}}  - \mathbf Q ^{\otimes ^{p}}  (  \mathbf I-\mathbf Q ) ^{\otimes ^{p}}              \\
    &= (\mathbf I-\mathbf Q ) ^{\otimes ^{p}}     -   (  \mathbf Q-\mathbf Q^{2} ) ^{\otimes ^{p}}.
    \end{split}
    \end{equation}

 It can be easily verified that $ \mathbf Q=\mathbf Q^{2} $ since $ \mathbf Q$ is a projection matrix, and thus
    \begin{equation}\label{leftright2}
    \mathbf {P}^{\bot}_{\mathbf S ^{\otimes ^{p}} } (\mathbf {P}^{\bot}_{\mathbf S} )^{\otimes ^{p}}=  (\mathbf I-\mathbf Q ) ^{\otimes ^{p}}  =  (\mathbf {P}^{\bot}_{\mathbf S} )^{\otimes ^{p}}.
    \end{equation}

    (\ref{leftright2}) implies that the projection of the matrix  $(\mathbf {P}^{\bot}_{\mathbf S} )^{\otimes ^{p}} $ in the  space spanned by the columns of   $ \mathbf {P}^{\bot}_{\mathbf S ^{\otimes ^{p}} } $ is still itself. Then we can conclude that  (\ref{lemmacontent}) holds.

%


  \end{proof}

  Furthermore, we can obtain the dimensionality of the space spanned by $(\mathbf {P}^{\bot}_{\mathbf S} )^{\otimes ^{p}} $  and    $ \mathbf {P}^{\bot}_{\mathbf S ^{\otimes ^{p}} } $.
Theoretically, the rank of a matrix $ \mathbf S $ can be defined as the dimensionality of the range of $ \mathbf S $, which follows
    \begin{equation}\label{rankrange}
  \rm rank(\mathbf S)= \rm dim  [R(\mathbf S ) ].
  \end{equation}

 According to (\ref{KPrank}), we can have
  \begin{equation}\label{complementKP}
 \rm dim[\rm R(  (\mathbf {P}^{\bot}_{\mathbf S} )^{\otimes ^{\it p}}    ) ] = \rm rank((\mathbf {P}^{\bot}_{\mathbf S})^{\otimes ^{\it p}}  )  = [\rm rank(\mathbf {P}^{\bot}_{\mathbf S}  )] ^{\it p}.
  \end{equation}

  A vector space $\rm V $ is the direct sum of the subspace $\rm W $ and its orthogonal complement space $ \rm W ^{\bot}$ and the dimensionality will satisfy the following relationship
  \begin{equation}\label{directsum}
  \rm dim(W)+\rm dim(W ^{\bot}) =\rm dim(V).
  \end{equation}

 Assume that subspace $\rm W $ is  spanned by the columns of matrix $\mathbf S $ and therefore
  \begin{equation}\label{ranksum}
  \rm rank(\mathbf S)+\rm rank(\mathbf S ^{\bot}) =n.
  \end{equation}

  Combining (\ref{ranksum}) with (\ref{complementKP}), we can obtain
  \begin{equation}\label{rankOfcomplementKP}
  \rm dim[ \rm R(  (\mathbf {P}^{\bot}_{\mathbf S} )^{\otimes ^{\it p}}    )   ]=(n-m)^{\it p}.
  \end{equation}

  In a similar way ,we can have
  \begin{equation}\label{rankOfKPcomplement}
 \rm dim[ \rm R(  \mathbf {P}^{\bot}_{\mathbf S ^{\otimes ^{\it p}} }  )  ]= n^{\it p}-m^{\it p}.
  \end{equation}

   Since $ n  \ge m > 0$ and $ p \in N^{+} $, according to the binomial theorem, the following  inequality can be deduced
  \begin{equation}\label{inequality}
 (n-m)^{p}  \le  n^{p}-m^{p},
  \end{equation}
  which is consistent with the conclusion  in  Lemma \ref{lemma1}.


Now, we    reconsider   (\ref{updatedtensor})   and  (\ref{oldmethod}) in the  $ L^{3}$-dimensional  space, and  we can have
  \begin{equation}\label{npsa1d}
 \mathbf {P}^{\bot}_{\mathbf u ^{\otimes ^{3}} } \cdot \mathbf s  \in \rm R(  \mathbf {P}^{\bot}_{\mathbf u ^{\otimes ^{3}} }  )
  \end{equation}
 and
   \begin{equation}\label{psa1d}
(\mathbf {P}^{\bot}_{\mathbf u} )^{\otimes ^{3}} \cdot  \mathbf s  \in \rm R(  (\mathbf {P}^{\bot}_{\mathbf u} )^{\otimes ^{3}}    ),
  \end{equation}
where we utilize the property (\ref{tensorkron}).

  Based on Lemma \ref{lemma1}, it always holds that   $\rm R(  (\mathbf {P}^{\bot}_{\mathbf u} )^{\otimes ^{3}} )  \subseteq  \rm R(  \mathbf {P}^{\bot}_{\mathbf u ^{\otimes ^{3}} }  ) $. This  implies that  the  strategy of NPSA can  enlarge the search space of the eigenvector to be determined  in each unit, instead of being restricted in the orthogonal complement space of the previous one  as in  PSA.   Meanwhile, similar to PSA, NPSA  can also  prevent the  solution from converging to the same one that has been determined before becaues of  the  use of the orthogonal complement operator in the  $ L^{3}$-dimensional space  given by  (\ref{step3}). 


\section{Experiments}
In this section, a series of experiments on  both simulated  and real multi/hyperspectral data  are conducted to evaluate the performance of   NPSA,   and several widely used  algorithms are compared meanwhile. All the algorithms are programmed and implemented in MATLAB R$2016$b on a laptop of $8$ GB RAM, Inter(R) Core (TM) i$5$-$4210$U CPU, @$1.70$GHZ.

\subsection{Experiments On Blind Image Separation}
To evaluate  the separation performance of NPSA, we apply it to the blind image separation (BIS) problem. BIS is an important application for ICA, and four algorithms designed for this problem are compared with our  method, including the original PSA \cite{PSA}, EcoICA (a third-order version of JADE)\cite{EcoICA}, subspace fitting method (SSF)\cite{SSF} and Distinctively
Normalized Joint Diagonalization (DNJD)\cite{DNJD} .

The aim of BIS is to estimate the mixing matrix, denoted by $\mathbf B$ (or its inverse matrix, i.e., the demixing matrix, denoted by $\mathbf U$) when only the mixed data is known. Here, three  gray-scale images with a size of  $ 256 \times 256 $ are selected as the source images, as shown in the first column in Fig~\ref{bisresult}. The mixing matrix $\mathbf B$ can be  generated by  the $\rm rand $ function in MATLAB software,  and  the mixed images are shown in the second column in Fig~\ref{bisresult}.  Then we apply these different algorithms to estimate the mixing (or the  demixing)   matrix and to obtain the separated images. The results  are shown in Fig~\ref{bisresult}. In order to ensure the reliability of the conclusion, we also conduct  the other four  combinations, in each of which we randomly select three different source images from each other.  Finally, several indices are further computed to evaluate  their  accuracy  as follows.


\subsubsection{ intersymbol-interference ($\rm ISI$)}
  This index  \cite{ISI}  is to measure the separation performance. After estimating the demixing matrix $\mathbf U$ (and the  mixing matrix  $\mathbf B$ is known), let $\mathbf P=\mathbf U\mathbf B$  , the ISI index is defined as
\begin{equation}\label{e20}
\rm ISI=\sum_{i=1}^L(\sum_{j=1}^L \frac {\vert \mathbf P_{ij}\vert^2 } {\max_k \vert \mathbf P_{ik}\vert^2}-1)+\sum_{j=1}^L(\sum_{i=1}^L \frac {\vert \mathbf P_{ij}\vert^2 } {\max_k \vert \mathbf P_{kj}\vert^2}-1).
\end{equation}

It is obvious that if $ \mathbf U=\mathbf B^{-1}$, $ \mathbf P$ is an identity matrix, so the ISI is equal to zero. The smaller the ISI is, the better the algorithm performs. We take the average value of 10
runs as the result.  The comparison between these  methods is listed in Table \ref{ISI}.

 As can be seen, NPSA performs better than the others in combination 1, 2, 4 and 5, especially in combination 5. EcoICA is slightly superior  to NPSA in combination 3.

\begin{figure*}[!htbp]
\begin{subfigure}[htbp!]{0.13\textwidth}
  \centering
  \subcaption*{source image}
  \includegraphics[width=1.0\textwidth]{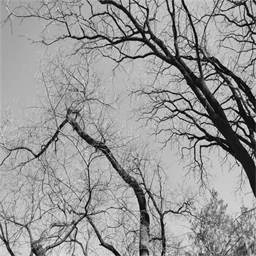}
  \label{r1}
\end{subfigure}
\hfill
\begin{subfigure}[htbp!]{0.13\textwidth}
  \centering
  \subcaption*{mixed image}
  \includegraphics[width=1.0\textwidth]{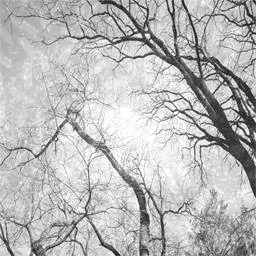}
  \label{m1}
\end{subfigure}
\hfill
\begin{subfigure}[htbp!]{0.13\textwidth}
  \centering
  \subcaption*{EcoICA}
  \includegraphics[width=1.0\textwidth]{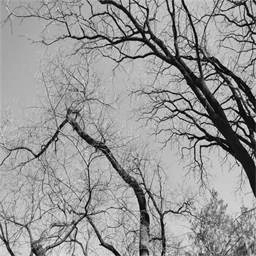}
 \label{ecobis1}
\end{subfigure}
\hfill
\begin{subfigure}[htbp!]{0.13\textwidth}
  \centering
   \subcaption*{DNJD}
  \includegraphics[width=1.0\textwidth]{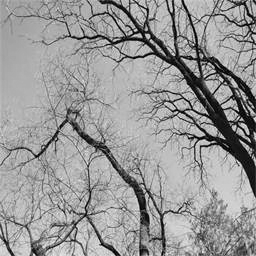}
 \label{dnjdbis1}
\end{subfigure}
\hfill
\begin{subfigure}[htbp!]{0.13\textwidth}
  \centering
  \subcaption*{SSF}
  \includegraphics[width=1.0\textwidth]{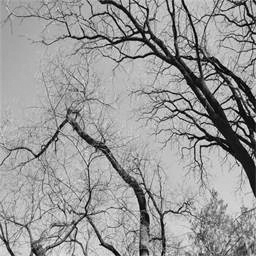}
    \label{ssfbis1}
\end{subfigure}
\hfill
\begin{subfigure}[htbp!]{0.13\textwidth}
  \centering
  \subcaption*{PSA}
  \includegraphics[width=1.0\textwidth]{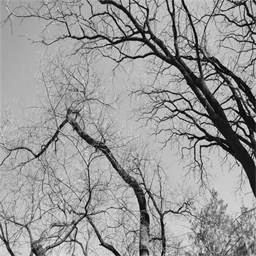}
  \label{psabis1}
\end{subfigure}
\hfill
\begin{subfigure}[htbp!]{0.13\textwidth}
  \centering
  \subcaption*{NPSA}
  \includegraphics[width=1.0\textwidth]{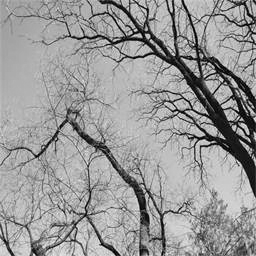}
  \label{npsabis1}
\end{subfigure}
\hfill

\begin{subfigure}[htbp!]{0.13\textwidth}
  \centering
  \includegraphics[width=1.0\textwidth]{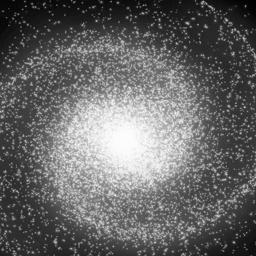}
\end{subfigure}
\hfill
\begin{subfigure}[htbp!]{0.13\textwidth}
  \centering
  \includegraphics[width=1.0\textwidth]{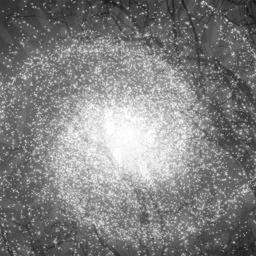}
\end{subfigure}
\hfill
\begin{subfigure}[htbp!]{0.13\textwidth}
  \centering
  \includegraphics[width=1.0\textwidth]{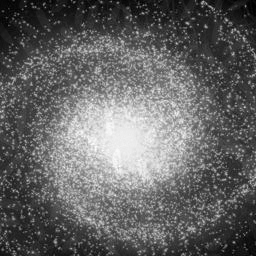}
\end{subfigure}
\hfill
\begin{subfigure}[htbp!]{0.13\textwidth}
  \centering
  \includegraphics[width=1.0\textwidth]{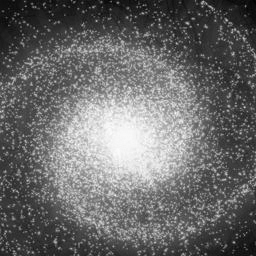}
\end{subfigure}
\hfill
\begin{subfigure}[htbp!]{0.13\textwidth}
  \centering
  \includegraphics[width=1.0\textwidth]{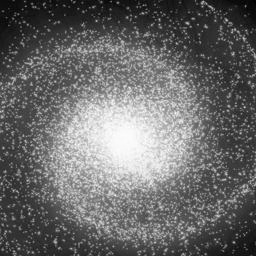}
\end{subfigure}
\hfill
\begin{subfigure}[htbp!]{0.13\textwidth}
  \centering
  \includegraphics[width=1.0\textwidth]{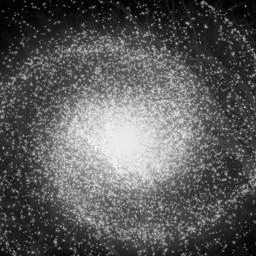}
\end{subfigure}
\hfill
\begin{subfigure}[htbp!]{0.13\textwidth}
  \centering
  \includegraphics[width=1.0\textwidth]{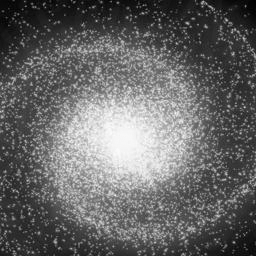}
\end{subfigure}
\hfill

\begin{subfigure}[htbp!]{0.13\textwidth}
  \centering
  \subcaption*{}
  \includegraphics[width=1.0\textwidth]{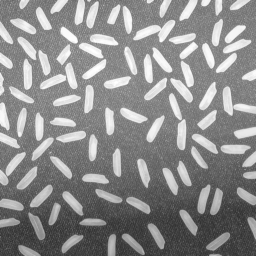}
\end{subfigure}
\hfill
\begin{subfigure}[htbp!]{0.13\textwidth}
  \centering
  \subcaption*{}
  \includegraphics[width=1.0\textwidth]{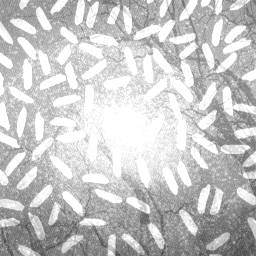}
\end{subfigure}
\hfill
\begin{subfigure}[htbp!]{0.13\textwidth}
  \centering
  \subcaption*{}
  \includegraphics[width=1.0\textwidth]{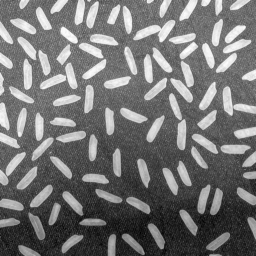}
\end{subfigure}
\hfill
\begin{subfigure}[htbp!]{0.13\textwidth}
  \centering
  \subcaption*{}
  \includegraphics[width=1.0\textwidth]{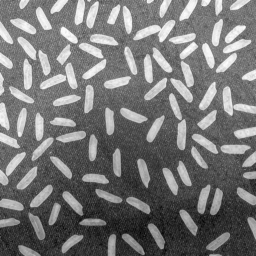}
\end{subfigure}
\hfill
\begin{subfigure}[htbp!]{0.13\textwidth}
  \centering
  \subcaption*{}
  \includegraphics[width=1.0\textwidth]{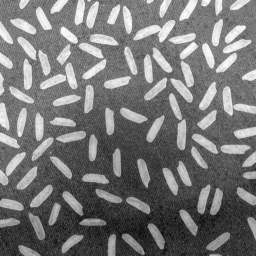}
\end{subfigure}
\hfill
\begin{subfigure}[htbp!]{0.13\textwidth}
  \centering
  \subcaption*{}
  \includegraphics[width=1.0\textwidth]{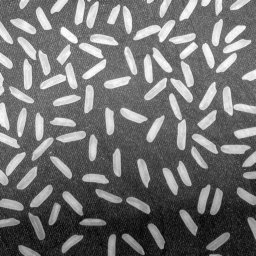}
\end{subfigure}
\hfill
\begin{subfigure}[htbp!]{0.13\textwidth}
  \centering
   \subcaption*{}
  \includegraphics[width=1.0\textwidth]{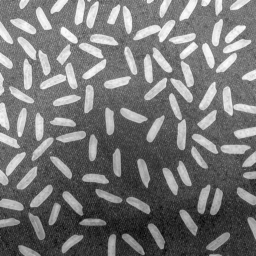}
\end{subfigure}
\hfill
\caption{\quad  
The results of NPSA, PSA, EcoICA, DNJD, and SSF. The first and second  column are the three source images and randomly-mixed images as the reference, respectively. }
\label{bisresult}
\end{figure*}

\begin{table}[t]
   \normalsize
    \centering
    \caption{ \\ Comparison of EcoICA, DNJD, SSF, PSA and NPSA for the ISI index  of five different combinations. An average result of ten runs is computed.}
\setlength{\tabcolsep}{1pt}
\renewcommand\arraystretch{1.3}
    \begin{tabular}{  c  c  c  c  c c }
		\hline
		Combination  & EcoICA & DNJD & SSF & PSA &NPSA   \\
		\hline
		1	& 0.0043 	& 0.2375 	&	0.1123  &  0.0061  & $\bf 0.0031 $ \\
        2   & 0.0309	& 0.3861    &  	0.0656  &  0.0363  & $\bf  0.0193 $  \\
        3 	&$\bf  0.0122 $  & 0.1353    & 	0.0179 &           0.0186   & 0.0151 	\\
        4  &0.0309	& 0.1119    &   0.1238   &	 0.0335 & $\bf  0.0196 $     \\
        5  &   0.0027  & 0.l923& 	0.0458  &   0.0035  &	\bf 0.0004     \\    
		  \hline
    \end{tabular}
      \label{ISI}
\end{table}

   \begin{table*}[!htbp]
    \centering

    \caption{ \\ Comparison of EcoICA, DNJD, SSF, PSA and NPSA for the TMSE index  of five different combinations. An average result of ten runs is computed.}
\setlength{\tabcolsep}{9pt}
    \renewcommand\arraystretch{1.5}
    \begin{tabular}{  c  c  c  c  c c }
		\hline
		Combination    & EcoICA & DNJD & SSF & PSA &NPSA  \\
		\hline
		 1 &    $ 2.6416\cdot 10^{-15}$ &	$2.5130 \cdot 10^{-11}$ &	$2.6011 \cdot 10^{-16}$ &  $6.7102 \cdot 10^{-15}$  & 	    $\bf 2.4090\cdot 10^{-16}$ \\
          2 &     $ 6.2017\cdot  10^{-13}$ &	$3.0427\cdot 10^{-11}$ &	$\bf 1.5112\cdot 10^{-15}$ &  $1.0104\cdot 10^{-12}$  & 	    $1.5827\cdot 10^{-15}$ \\
           3 &    $2.3501\cdot 10^{-14}$ &	$3.4753\cdot 10^{-11}$ &	$1.1480\cdot 10^{-12}$ &  $3.4730\cdot 10^{-14} $ & 	       $\bf  1.9547\cdot 10^{-14}$ \\
            4 &     $ 6.2017\cdot 10^{-13}$ &	$6.3559\cdot 10^{-12}$ &	$1.5100\cdot 10^{-15}$ &  $1.0025\cdot 10^{-12}$  & 	    $\bf 1.4332\cdot 10^{-15}$ \\
             5 &     $  2.4311\cdot 10^{-15}$ &	$1.2730\cdot 10^{-10}$ &	$6.2264\cdot 10^{-17}$ &  $8.6409\cdot 10^{-17} $ & 	    $\bf 6.1932\cdot              10^{-19}$ \\
        \hline
	\end{tabular}
 \label{TMSE}
\end{table*}

\subsubsection{ Total mean square error ($\rm TMSE $)}

\begin{table}[h]
    \centering
    \caption{  \\ Comparison of EcoICA, DNJD, SSF, PSA and NPSA for the correlation coefficient index   of five different combinations. An average result of ten runs is computed. }
\setlength{\tabcolsep}{6pt}  
\renewcommand\arraystretch{1.7}
    \begin{tabular}{c  c  c  c  c c  }  
		\hline
		Combination    & EcoICA & DNJD & SSF & PSA & NPSA  \\
		\hline
\multirowcell{3}{1} &0.9972 & 0.9230  & 0.9978 &	0.9969  & \bf {0.9997}    \\
  & 	\bf {0.9992}  &  	0.8596 &	0.9944 &\bf {0.9992}  &	\bf {0.9992}    \\
     &	\bf {1.0000}  &	0.9637  & 	0.9969 &	0.9998    &\bf { 1.0000 }  \\
        \hline
        \multirowcell{3}{2}    &	\bf {1.0000	}& 0.9611	& 0.9963    &	0.9990           & 0.9985  \\
 		& 0.9553 & 0.8848  & 	0.9872 & 0.9580   & \bf {0.9988}  \\
 	& 0.9988 & 	0.9517 & 	\bf {0.9989 }  & 0.9986  &\bf {0.9989}  \\
        \hline

        \multirowcell{3}{3}  	  &	0.9916 & 	0.9451 & 	0.9909   &  0.9902  & \bf {0.9920}    \\
	& \bf {1.0000}	& 0.9760  & 	0.9997  &  	0.9999  & \bf {1.0000 }  \\
 	 &  0.9977	 & 0.9252	  & 0.9871   &  0.9972	  & \bf {0.9995}         \\
        \hline

        \multirowcell{3}{4}  	 &	0.9988&	0.9623	&0.9972 &0.9987   &\bf {0.9989} \\
	&	0.9553&	0.8300&	0.9632     &0.9611    &\bf {0.9989}  \\
		&\bf {1.0000}	&0.9466	& 0.9935  &0.9992       &0.9983  \\
        \hline

        \multirowcell{3}{5}  		&\bf {1.0000} &	0.6621	&  0.9616 & 0.9998  & \bf {1.0000} \\
	&	\bf {1.0000}&	0.9906&	 0.9922   &0.9998   &\bf {1.0000}  \\
		&0.9972&	0.9156	&0.9968  &0.9966      &\bf {0.9997}     \\
        \hline
    \end{tabular}
      \label{SVC}
\end{table}

 \begin{figure}[t]
  \centering

  \includegraphics[width=0.25\textwidth]{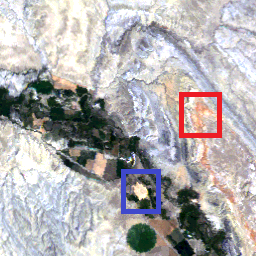}  \\
  \caption{\quad True color image of the multispectral data.(R:band $3$, G:band $2$, B:band $1$). }
  \label{msi}
\end{figure}

  To compute the error  between the source image $\mathbf I$  and the separated image $ \hat {\mathbf I}$, we use the mean square error (MSE) index. The images are first normalized into one length to eliminate the influence of the magnitude of the images, which can be implemented by $\mathbf I= \mathbf I / \Vert \mathbf I \Vert_{\mathrm F} $.  Then, assume that the number of the pixels is $ N $, the  MSE index can be calculated by
\begin{equation}\label{e21}
\rm MSE=\frac {1}{N} \Vert  \mathbf I   -\hat {\mathbf I} \Vert_{F}^{2},
\end{equation}
after computing the respective error of the  three images, we then compute the total mean square error (TMSE)
\begin{equation}\label{e22}
\rm TMSE=\frac {1}{3} \sum_{\it{i}=1}^3 \rm MSE_{ \it{i} }^{2}.
\end{equation}
	
  Similar to the ISI index, the smaller the TMSE is, the better the algorithm performs. We take the average value of 10
runs as the result.  The comparison between these methods is listed in Table \ref{TMSE}.

 The NPSA algorithm has the smallest TMSE in combination 1, 2, 4 and 5, while SSF is slightly  better than  NPSA in combination 3.
\subsubsection{correlation coefficient}

Both ISI and TMSE mainly focus on the overall performance. Here, we use the correlation coefficient to  measure the similarity between  $ \mathbf I$ and  $ \hat {\mathbf I}$, in detail.   The two images are first vectorized into the vector $ \mathbf i $ and  $ \hat { \mathbf i} $ (Note that the normalization operation used in TMSE index have no impact on the final result since the correlation coefficient mainly measure the angle between the two vectors). Then the  correlation coefficient can be calculated by
\begin{equation}\label{e23}
\rho=\frac {\mathbf i  \cdot \hat { \mathbf i}}  {\Vert \mathbf i \Vert \cdot  \Vert \hat { \mathbf i} \Vert  }.
\end{equation}

   The results of the five combinations are listed in  Table \ref{SVC}. It can be found that the correlation coefficient  of   NPSA  is higher than that of the others in most of the  results and has a more stable performance meanwhile.

 Eventually, combining the comparison results of several indices, we can conclude that the proposed algorithm obtains a more accurate and  robust  performance in terms of BIS application.

\subsection{Experiments On Multispectral Data}


In this experiment, a $30$-m resolution Landsat-$5$ image embedded in the Environment for Visualizing Images  (ENVI) software is selected to evaluate the performance of NPSA. Three algorithms, including  the classical FastICA, PSA and EcoICA, are compared in the following experiments. The other two algorithms compared in the BIS experiments are not  included  since they have either an unstable performance or   an unbearable computational complexity  for real multi/hyperspectral data.

The dataset contains six
bands in $30$-m resolution, with band numbers $1$-$5$ and  $7$
(ranging from $0.482$ to $2.2$ um). Band $6 $ is a thermal band
($10.40$-$12.50$ um) with a much lower spatial resolution ($120$ m).
Thus, it is usually not included.  A subscene with a $256 \times 256$ pixel size is selected as the test data and  the true color image is shown in Fig.~\ref{msi}. In this area, the main land cover types include vegetation, water, bareland, etc.

In our experiment, we set the number of independent components to $p=6$, i.e., we select the full-band image to conduct the test. Eventually, we can obtain $6$ principal skewness components of the image  and the results are shown in Fig.~\ref{msiresult}.  IC$1$-$3$ are almost the  same, which  correspond to three main land cover types: vegetation, bareland and water, respectively. By inspection, the result of NPSA in  Fig.~\ref{npsa_msi5} (IC$2$) and Fig.~\ref{npsa_msi6} (IC$6$) is  also superior to that of the other algorithms.

It is  worth taking  IC$2$ and IC$6$ as an example to compare the differences of these algorithms once again.  IC$2$ mainly corresponds to  the cultivated  farmland (e.g., framed in the blue rectangular), while  IC$6$  corresponds to  the uncultivated bareland (e.g., framed in the red rectangular).  Their spectral curves are  shown in Fig.~\ref{msispec}.
 Because of the orthogonal constraint  in PSA, EcoICA and FastICA, for the objects with similar spectra,  there is often only one of them   being highlighted. When the farmland is extracted as the independent component  in IC$2$, the uncultivated bareland will be suppressed in the later component, as shown in Fig.~\ref{fastica_msi6}, Fig.~\ref{eco_msi6} and Fig.~\ref{psa_msi6}.
 While in NPSA, because the restriction of orthogonal constraint is relaxed,  the object that is spectrally similar to the previous components    can still be  likely to be detected.   As shown in Fig.~\ref{npsa_msi6}, the uncultivated bareland   can  still be efficiently extracted  in NPSA.

 \begin{figure*}[t]
\begin{subfigure}[htbp!]{0.15\textwidth}
  \centering
  \includegraphics[width=1.0\textwidth]{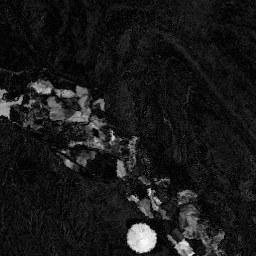}
  \subcaption{IC1 of FastICA}
  \label{fastica_msi1}
\end{subfigure}
\hfill
\begin{subfigure}[htbp!]{0.15\textwidth}
  \centering
  \includegraphics[width=1.0\textwidth]{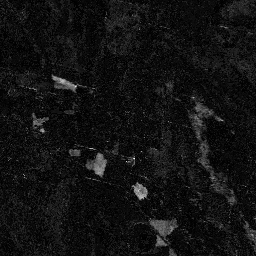}
   \subcaption{IC2 of FastICA}
  \label{fastica_msi2}
\end{subfigure}
\hfill
\begin{subfigure}[htbp!]{0.15\textwidth}
  \centering
  \includegraphics[width=1.0\textwidth]{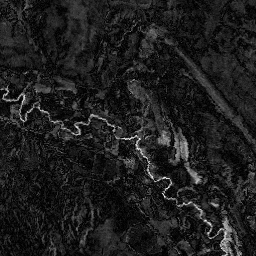}
   \subcaption{IC3 of FastICA}
  \label{fastica_msi3}
\end{subfigure}
\hfill
\begin{subfigure}[htbp!]{0.15\textwidth}
  \centering
  \includegraphics[width=1.0\textwidth]{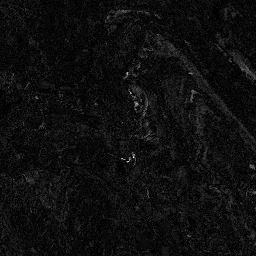}
   \subcaption{IC4 of FastICA}
  \label{fastica_msi4}
\end{subfigure}
\hfill
\begin{subfigure}[htbp!]{0.15\textwidth}
  \centering
  \includegraphics[width=1.0\textwidth]{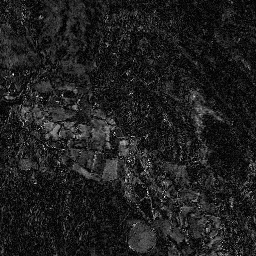}
   \subcaption{IC5 of FastICA}
  \label{fastica_msi5}
\end{subfigure}
\hfill
\begin{subfigure}[htbp!]{0.15\textwidth}
  \centering
  \includegraphics[width=1.0\textwidth]{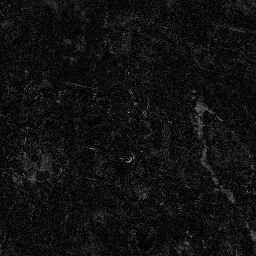}
   \subcaption{IC6 of FastICA}
  \label{fastica_msi6}
\end{subfigure}

\begin{subfigure}[htbp!]{0.15\textwidth}
  \centering
  \includegraphics[width=1.0\textwidth]{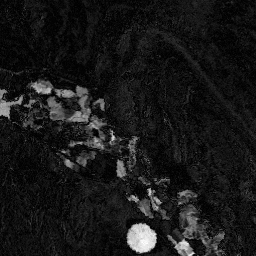}
   \subcaption{IC1 of EcoICA}
  \label{eco_msi1}
\end{subfigure}
\hfill
\begin{subfigure}[htbp!]{0.15\textwidth}
  \centering
  \includegraphics[width=1.0\textwidth]{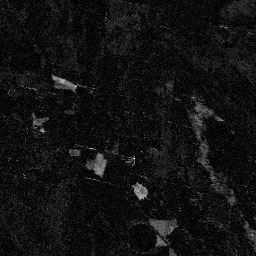}
   \subcaption{IC2 of EcoICA}
  \label{eco_msi2}
\end{subfigure}
\hfill
\begin{subfigure}[htbp!]{0.15\textwidth}
  \centering
  \includegraphics[width=1.0\textwidth]{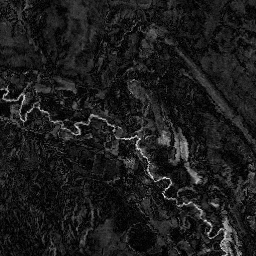}
   \subcaption{IC3 of EcoICA}
  \label{eco_msi3}
\end{subfigure}
\hfill
\begin{subfigure}[htbp!]{0.15\textwidth}
  \centering
  \includegraphics[width=1.0\textwidth]{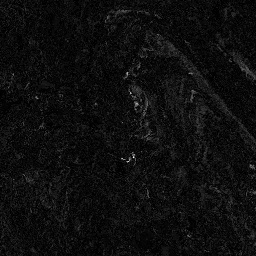}
   \subcaption{IC4 of EcoICA}
  \label{eco_msi4}
\end{subfigure}
\hfill
\begin{subfigure}[htbp!]{0.15\textwidth}
  \centering
  \includegraphics[width=1.0\textwidth]{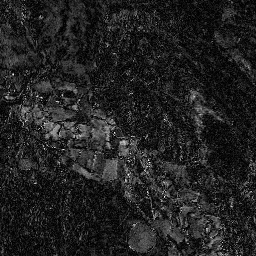}
   \subcaption{IC5 of EcoICA}
  \label{eco_msi5}
\end{subfigure}
\hfill
\begin{subfigure}[htbp!]{0.15\textwidth}
  \centering
  \includegraphics[width=1.0\textwidth]{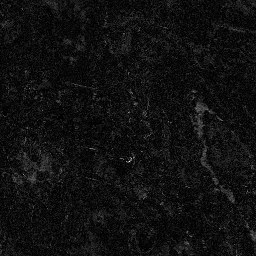}
   \subcaption{IC6 of EcoICA}
  \label{eco_msi6}
\end{subfigure}

\begin{subfigure}[htbp!]{0.15\textwidth}
  \centering
  \includegraphics[width=1.0\textwidth]{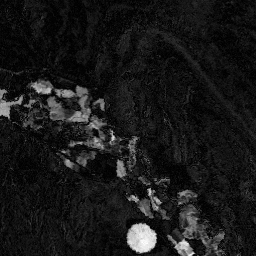}
   \subcaption{IC1 of PSA}
 \label{psa_msi1}
\end{subfigure}
\hfill
\begin{subfigure}[htbp!]{0.15\textwidth}
  \centering
  \includegraphics[width=1.0\textwidth]{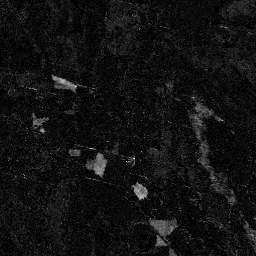}
   \subcaption{IC2 of PSA}
  \label{psa_msi2}
\end{subfigure}
\hfill
\begin{subfigure}[htbp!]{0.15\textwidth}
  \centering
  \includegraphics[width=1.0\textwidth]{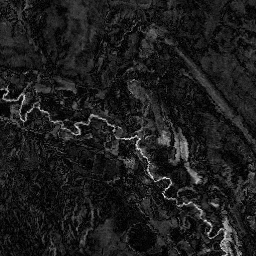}
   \subcaption{IC3 of PSA}
  \label{psa_msi3}
\end{subfigure}
\hfill
\begin{subfigure}[htbp!]{0.15\textwidth}
  \centering
  \includegraphics[width=1.0\textwidth]{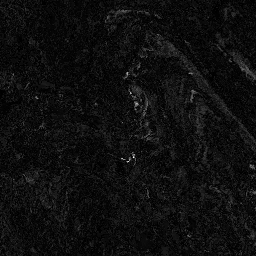}
   \subcaption{IC4 of PSA}
  \label{psa_msi4}
\end{subfigure}
\hfill
\begin{subfigure}[htbp!]{0.15\textwidth}
  \centering
  \includegraphics[width=1.0\textwidth]{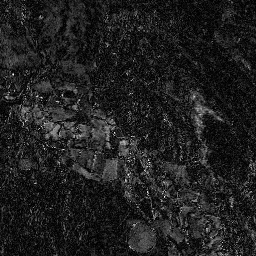}
   \subcaption{IC5 of PSA}
  \label{psa_msi5}
\end{subfigure}
\hfill
\begin{subfigure}[htbp!]{0.15\textwidth}
  \centering
  \includegraphics[width=1.0\textwidth]{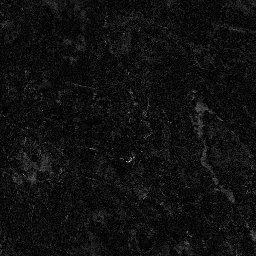}
   \subcaption{IC6 of PSA}
  \label{psa_msi6}
\end{subfigure}

\begin{subfigure}[htbp!]{0.15\textwidth}
  \centering
  \includegraphics[width=1.0\textwidth]{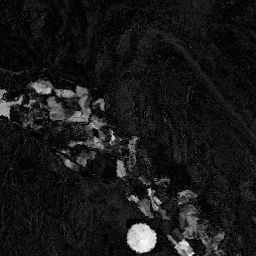}
   \subcaption{IC1 of NPSA}
  \label{npsa_msi1}
\end{subfigure}
\hfill
\begin{subfigure}[htbp!]{0.15\textwidth}
  \centering
  \includegraphics[width=1.0\textwidth]{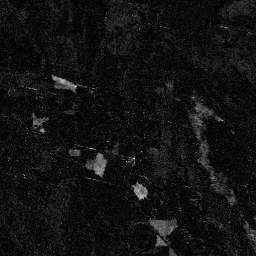}
   \subcaption{IC2 of NPSA}
  \label{npsa_msi2}
\end{subfigure}
\hfill
\begin{subfigure}[htbp!]{0.15\textwidth}
  \centering
  \includegraphics[width=1.0\textwidth]{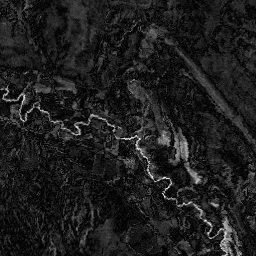}
   \subcaption{IC3 of NPSA}
  \label{npsa_msi3}
\end{subfigure}
\hfill
\begin{subfigure}[htbp!]{0.15\textwidth}
  \centering
  \includegraphics[width=1.0\textwidth]{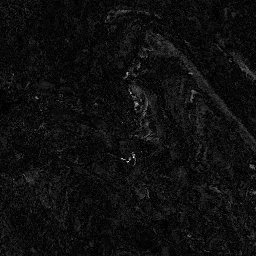}
   \subcaption{IC4 of NPSA}
  \label{npsa_msi4}
\end{subfigure}
\hfill
\begin{subfigure}[htbp!]{0.15\textwidth}
  \centering
  \includegraphics[width=1.0\textwidth]{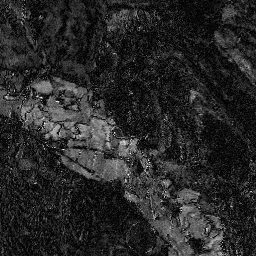}
   \subcaption{IC5 of NPSA}
  \label{npsa_msi5}
\end{subfigure}
\hfill
\begin{subfigure}[htbp!]{0.15\textwidth}
  \centering
  \includegraphics[width=1.0\textwidth]{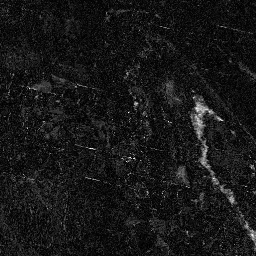}
   \subcaption{IC6 of NPSA}
  \label{npsa_msi6}
\end{subfigure}
\caption{\quad The results  of FastICA, EcoICA, PSA and NPSA for the multispectral image.  }
\label{msiresult}
\end{figure*}

To further demonstrate the advantage of NPSA, we conduct a quantitative comparison by calculating  the skewness value of each independent component,  and the results are plotted in Fig.~\ref{msivalue}.  As can be seen, the skewness value of NPSA in  each independent component is equal to or larger than that of the others, which implies that NPSA can obtain the more reasonable maxima  in each unit and has a better overall performance.

\begin{figure}[t]
  \centering
  \includegraphics[width=0.3\textwidth]{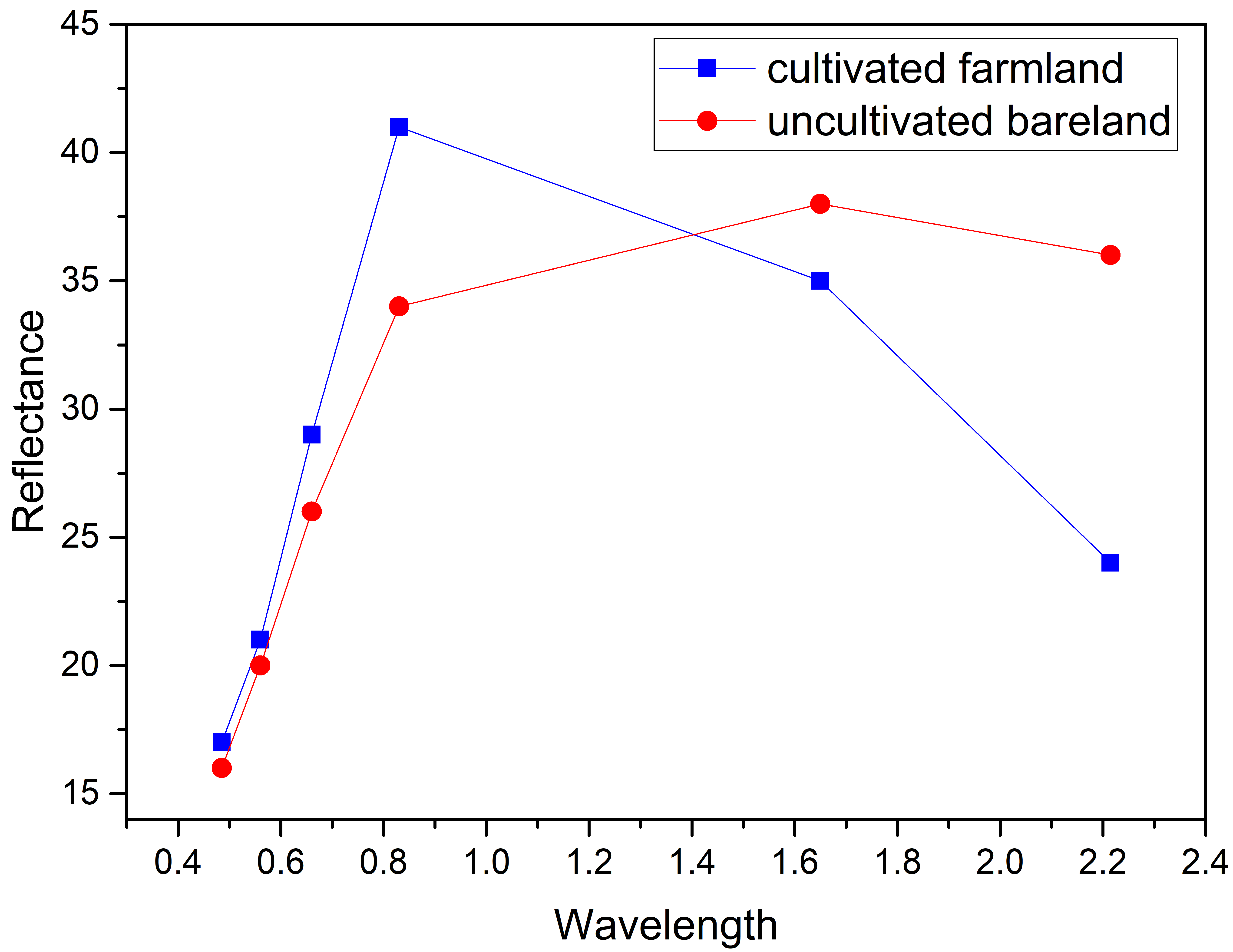}  \\
  \caption{\quad The spectral comparison for the framland and  the uncultivated bareland. }
  \label{msispec}
\end{figure}


\begin{figure}[t]

    \centering
  \includegraphics[width=0.3\textwidth]{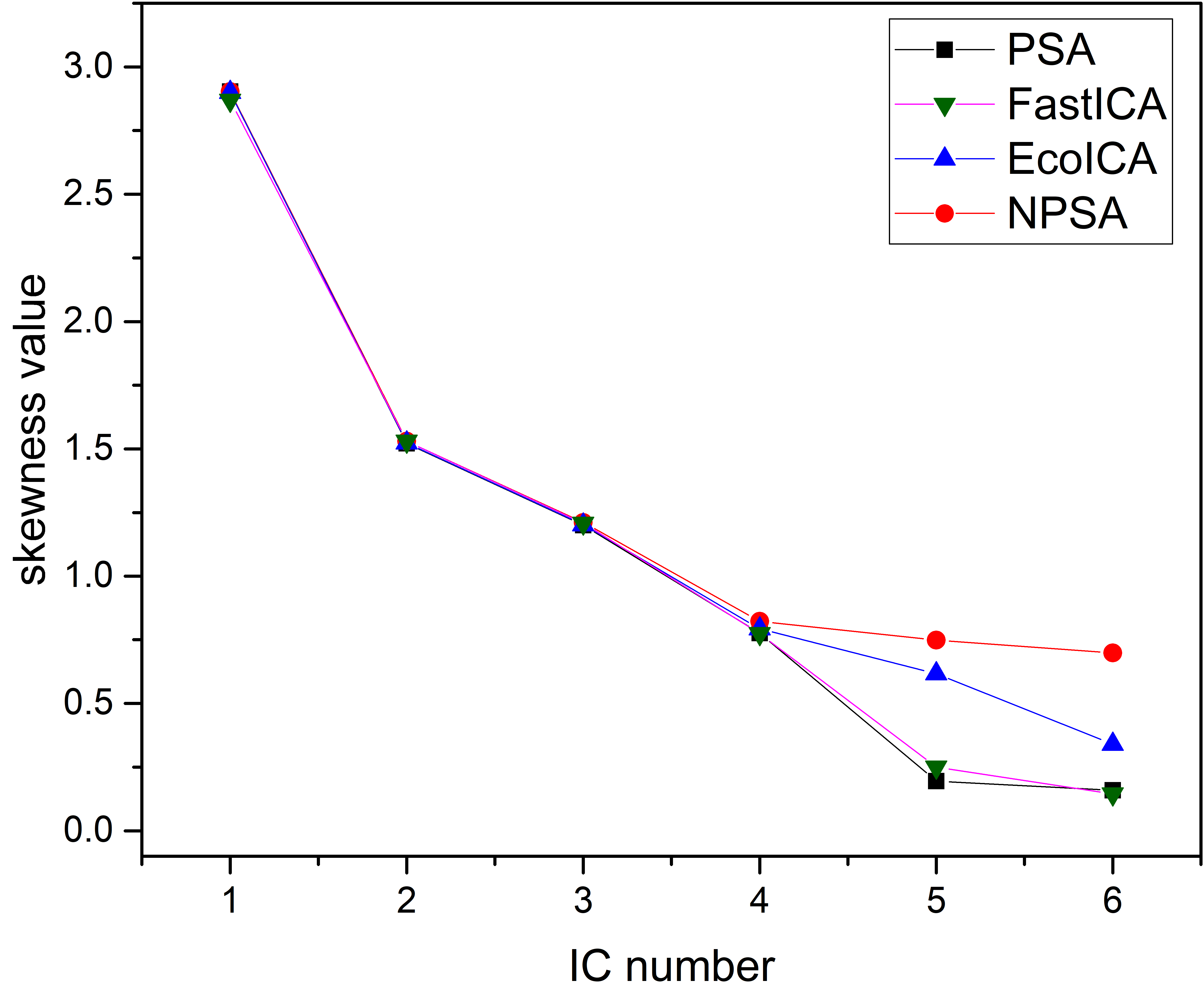}  \\
  \caption{\quad Skewness  comparison of FastICA, EcoICA, PSA and NPSA  for the multispectral image.}
  \label{msivalue}
\end{figure}

\begin{figure}[t]
    \centering
  \includegraphics[width=0.23\textwidth]{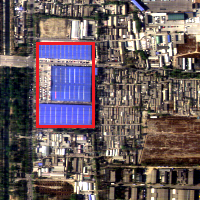}  \\
  \caption{ \quad   True color image of the hyperspectral data.
        {(R:$620$nm, G:$ 559$nm, B:$473$nm) }. }
   \label{hsi}
\end{figure}

\begin{figure}[t]
\begin{subfigure}[htbp!]{0.15\textwidth}
  \centering
  \includegraphics[width=1\textwidth]{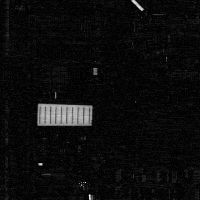}
     \subcaption{IC1 of FastICA  \\ ($\rm skew =4.7826 $) }
     \label{fastica1}
\end{subfigure}
\hfill
\begin{subfigure}[htbp!]{0.15\textwidth}
  \centering
  \includegraphics[width=1\textwidth]{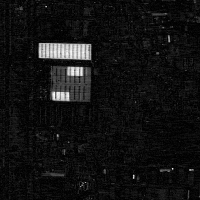}
     \subcaption{IC2 of FastICA  \\ ($\rm skew =4.2132 $) }
     \label{fastica2}
\end{subfigure}
\hfill
\begin{subfigure}[htbp!]{0.15\textwidth}
  \centering
  \includegraphics[width=1\textwidth]{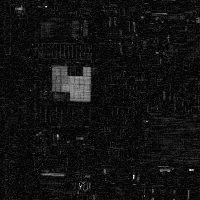}
  \subcaption{IC3 of FastICA  \\ ($\rm skew =1.7656 $) }
  \label{fastica3}
\end{subfigure}

\begin{subfigure}[htbp!]{0.15\textwidth}
  \centering
  \includegraphics[width=1\textwidth]{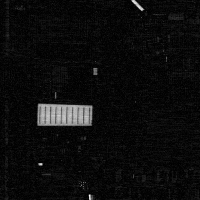}
  \subcaption{IC1 of EcoICA  \\ ($\rm skew =4.8943 $) }
  \label{eco1}
\end{subfigure}
\hfill
\begin{subfigure}[htbp!]{0.15\textwidth}
  \centering
  \includegraphics[width=1\textwidth]{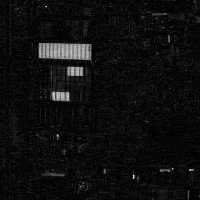}
  \subcaption{IC2 of EcoICA  \\($\rm skew =4.1186 $) }
  \label{eco2}
\end{subfigure}
\hfill
\begin{subfigure}[htbp!]{0.15\textwidth}
  \centering
  \includegraphics[width=1\textwidth]{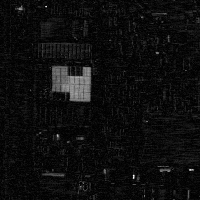}
  \subcaption{IC3 of EcoICA  \\ ($\rm skew =2.8699 $)}
  \label{eco3}
\end{subfigure}

\begin{subfigure}[htbp!]{0.15\textwidth}
  \centering
  \includegraphics[width=1\textwidth]{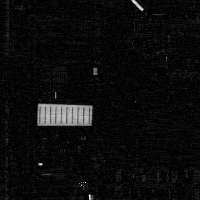}
  \subcaption{IC1 of PSA  \\ ($\rm skew =4.7648 $) }
  \label{psa1}
\end{subfigure}
\hfill
\begin{subfigure}[htbp!]{0.15\textwidth}
  \centering
  \includegraphics[width=1\textwidth]{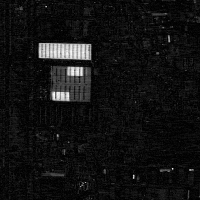}
  \subcaption{IC2 of PSA \\ ($\rm skew =4.2108 $) }
  \label{psa2}
\end{subfigure}
\hfill
\begin{subfigure}[htbp!]{0.15\textwidth}
  \centering
  \includegraphics[width=1\textwidth]{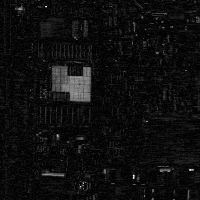}
  \subcaption{IC3 of PSA  \\ ($\rm skew =2.1952 $) }
  \label{psa3}
\end{subfigure}

\begin{subfigure}[htbp!]{0.15\textwidth}
  \centering
  \includegraphics[width=1\textwidth]{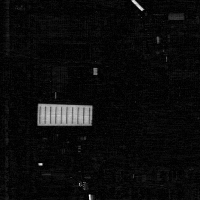}
  \subcaption{IC1 of NPSA  \\ ($\rm skew =4.9410 $) }
  \label{npsa1}
\end{subfigure}
\hfill
\begin{subfigure}[htbp!]{0.15\textwidth}
  \centering
  \includegraphics[width=1\textwidth]{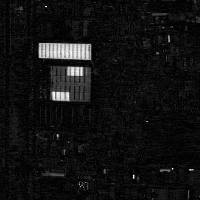}
  \subcaption{IC2 of NPSA  \\ ($\rm skew =4.3176 $) }
  \label{npsa2}
\end{subfigure}
\hfill
\begin{subfigure}[htbp!]{0.15\textwidth}
  \centering
  \includegraphics[width=1\textwidth]{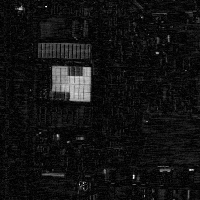}
  \subcaption{IC3 of NPSA  \\ ($\rm skew =2.9675 $) }
  \label{npsa3}
\end{subfigure}
\caption{\quad Some results  of FastICA, EcoICA, PSA and NPSA for the hyperspectral image. The  corresponding skewness value of each IC is also attached.}
\label{hsiresult}
\end{figure}


\subsection{Experiments On Hyperspectral Data}

Here, the hyperspectral image  data we used to test the method is  from  the Operational Modular Imaging Spectrometer-II, which were acquired by the Aerial Photogrammetry
and Remote Sensing Bureau in Beijing, China, in $2000$.
It includes $64$ bands from visible to thermal infrared with about $3$-m spatial
resolution and $10$-nm spectral resolution and has  $200 \times 200$ pixels in each band.  Since the signal-to-noise ratio is low in bands $55$-$60$, we select  the remaining $58$ bands as our test data. The true color image is shown in Fig.~\ref{hsi}. The red rectangular framed in Fig.~\ref{hsi} is an area of blue-painted roof. After a field investigation, it was found that the roof was made from three different materials although there is no obvious difference in the visible band.

 We display the results of feature extraction in Fig.~\ref{hsiresult} and plot the curves of skewness   in Fig.~\ref{hsivalue}. All  algorithms can  automatically extract the three different materials as the  independent components. 
The skewness  comparison both attached in Fig.~\ref{hsiresult} for the three ICs and in Fig.~\ref{hsivalue} for all the ICs can  also   demonstrate that NPSA outperforms the other  three algorithms. It can be observed that in Fig.~\ref{hsivalue},  the skewness of IC$2$ and IC$6$  extracted by  EcoICA is slightly larger than that of NPSA. However, we can still conclude that NPSA has the better overall performance. 


\begin{figure}[t]
\begin{subfigure}[htbp!]{0.23\textwidth}
  \centering
  \includegraphics[width=1\textwidth]{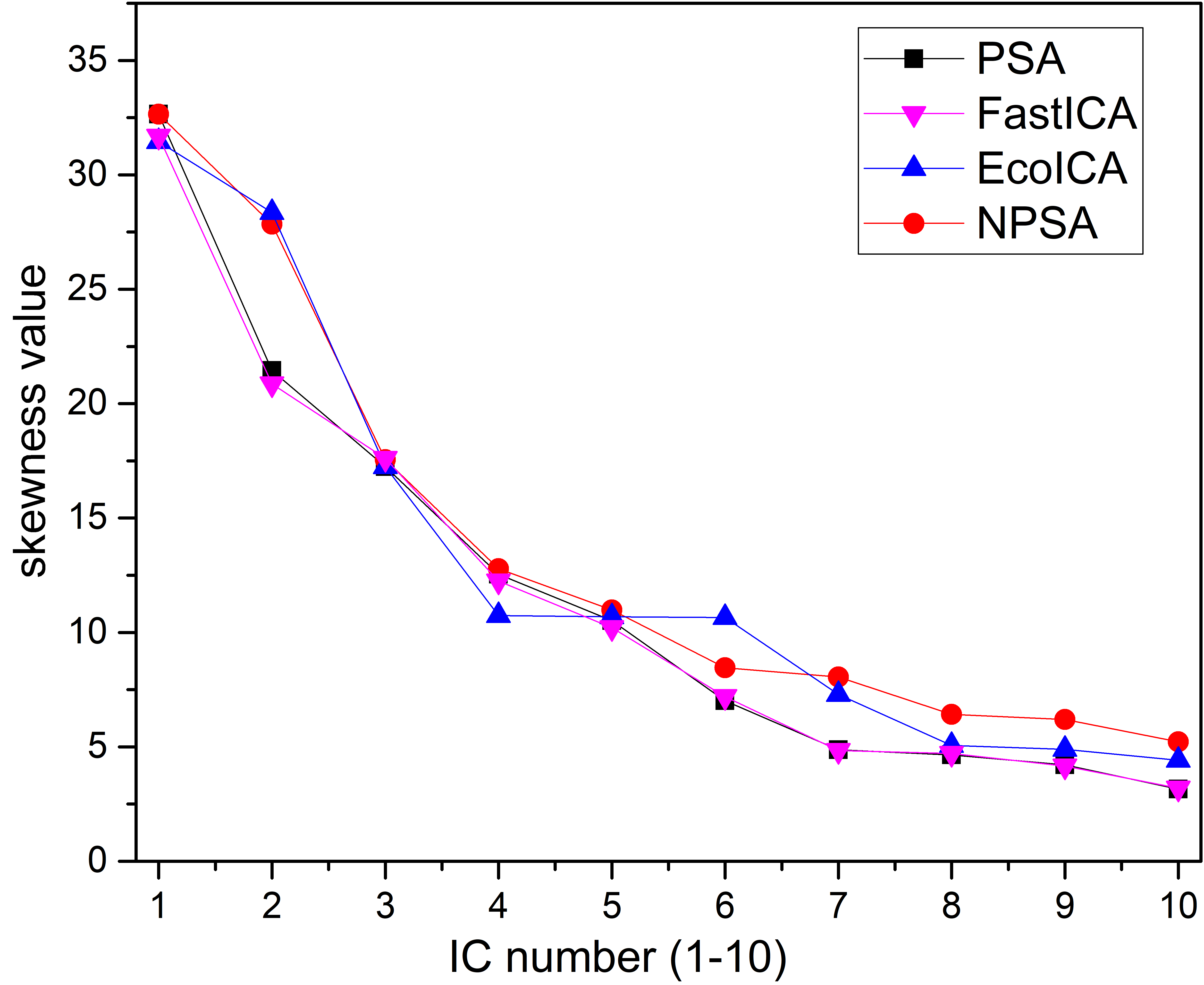}
     \label{sv1}
\end{subfigure}
\begin{subfigure}[htbp!]{0.23\textwidth}
  \centering
  \includegraphics[width=1\textwidth]{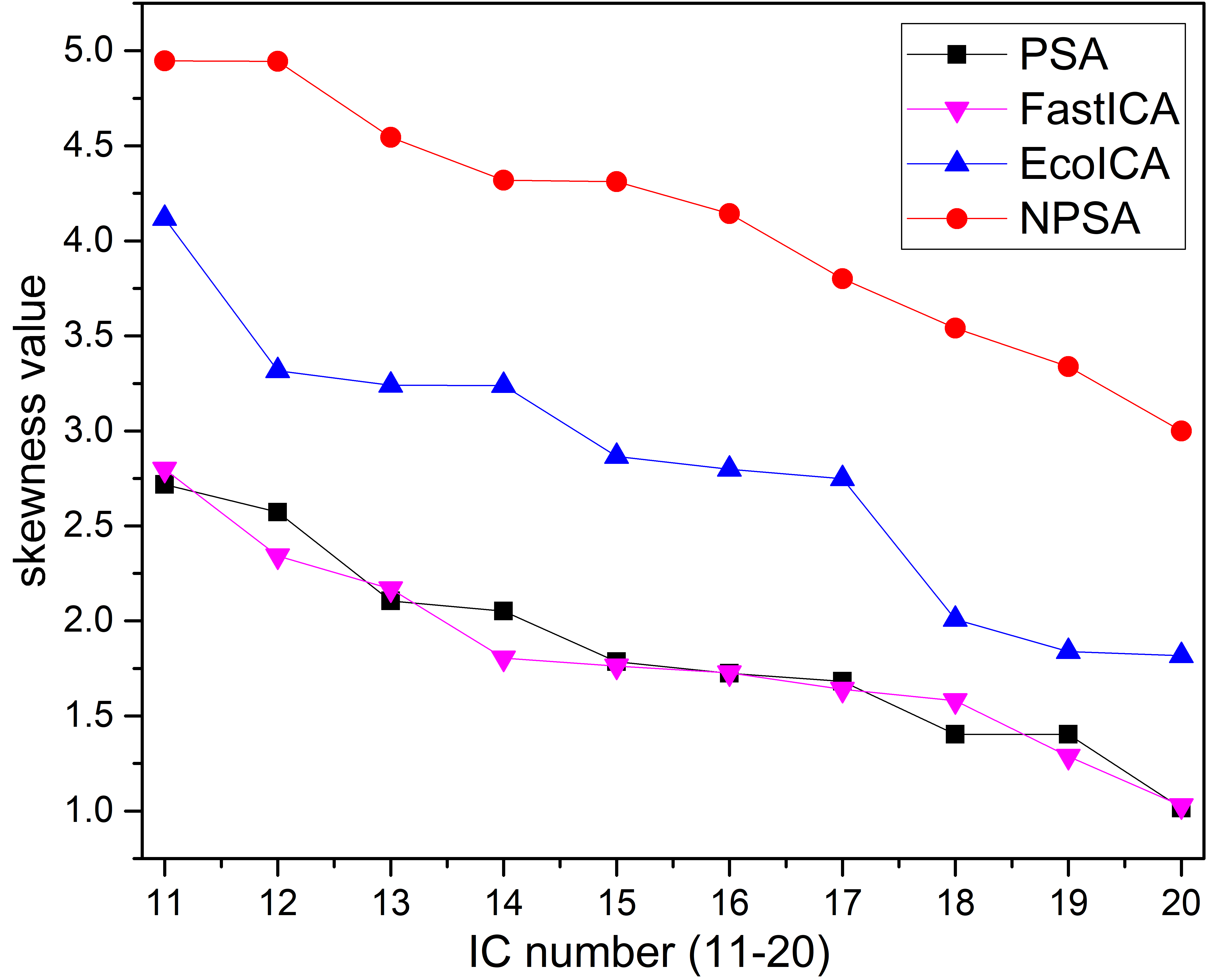}
     \label{sv2}
\end{subfigure}
\hfill

\begin{subfigure}[htbp!]{0.23\textwidth}
  \centering
  \includegraphics[width=1\textwidth]{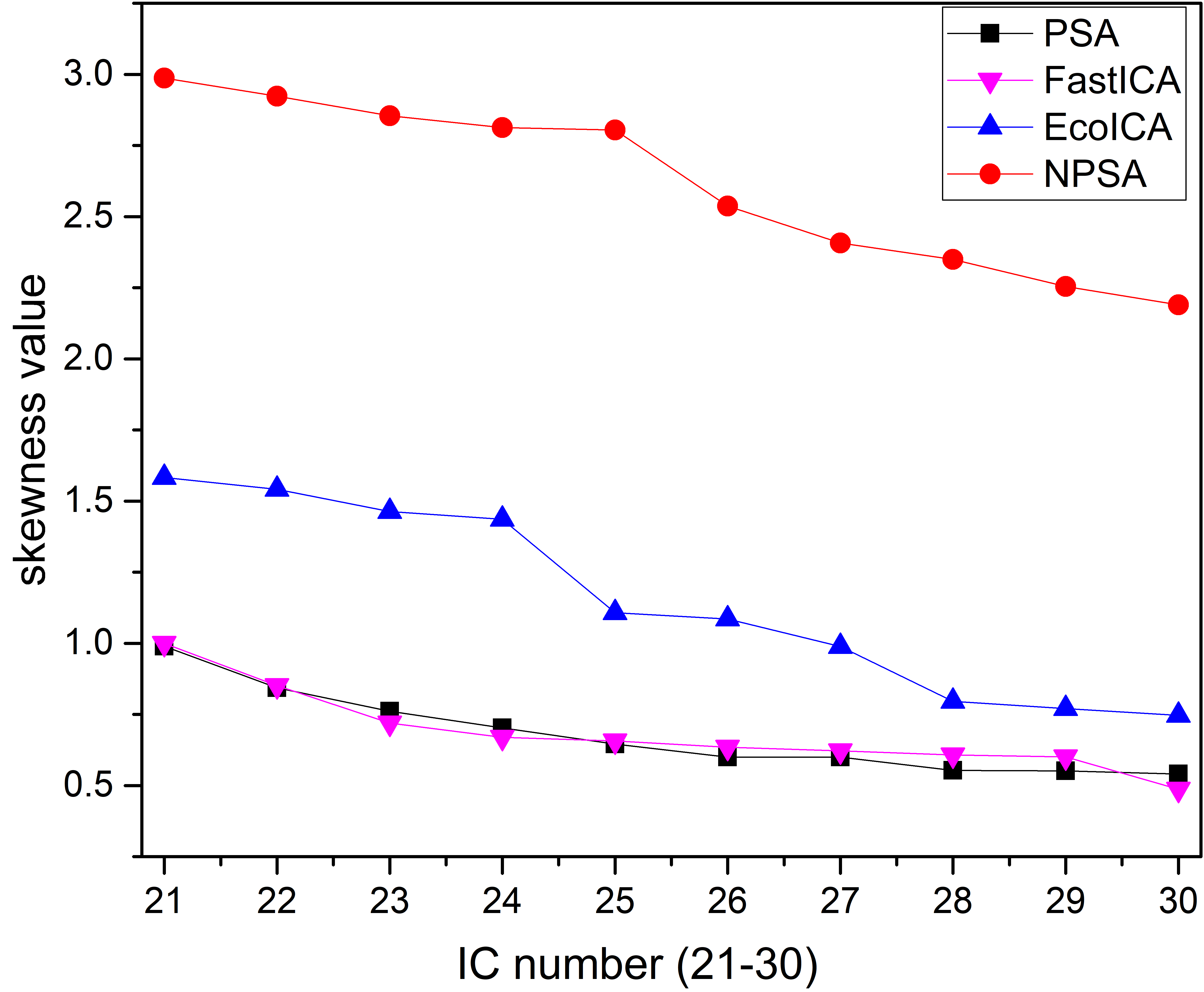}
  \label{sv3}
\end{subfigure}
\begin{subfigure}[htbp!]{0.23\textwidth}
  \centering
  \includegraphics[width=1\textwidth]{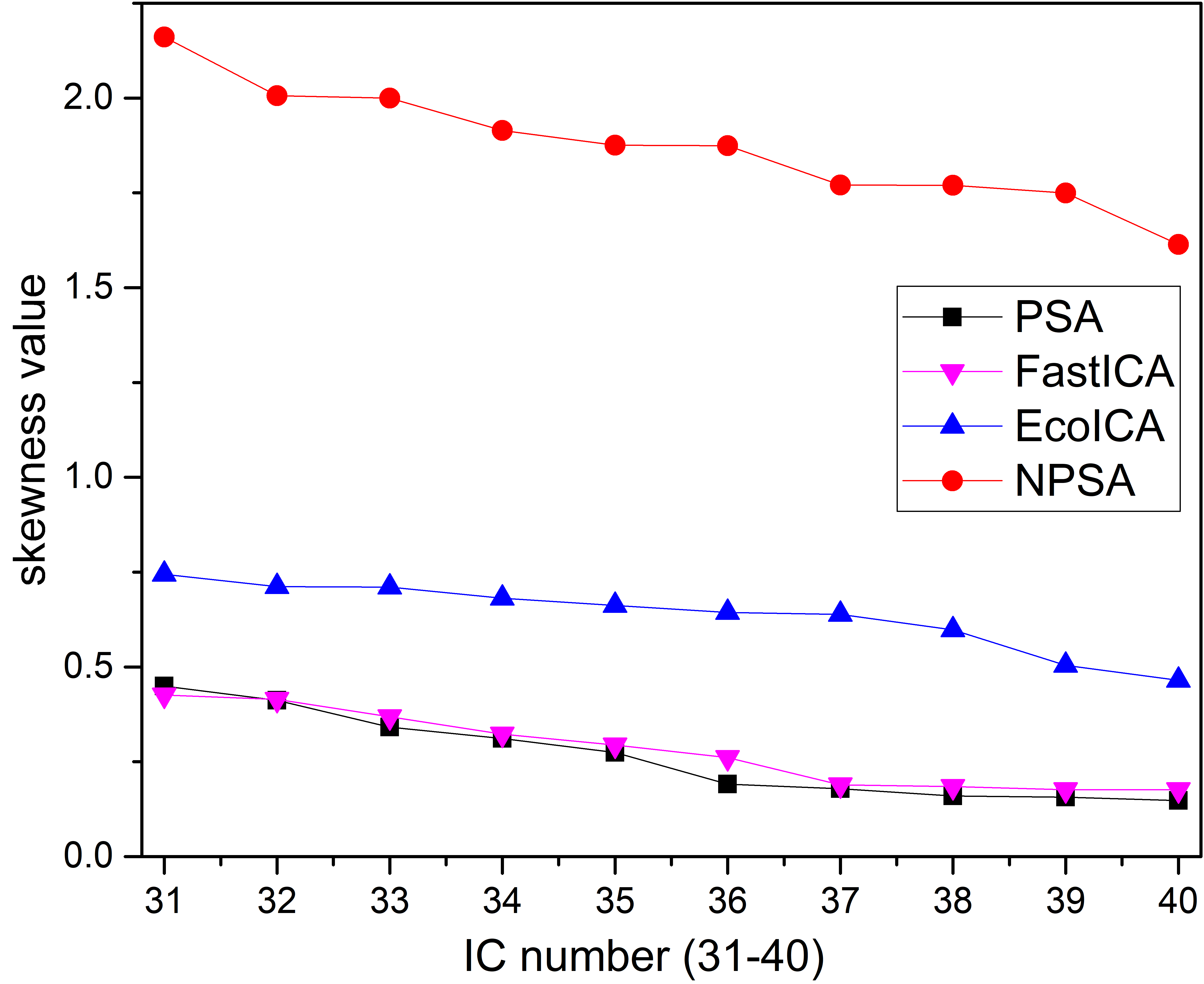}
  \label{sv4}
\end{subfigure}
\hfill

\begin{subfigure}[htbp!]{0.23\textwidth}
  \centering
  \includegraphics[width=1\textwidth]{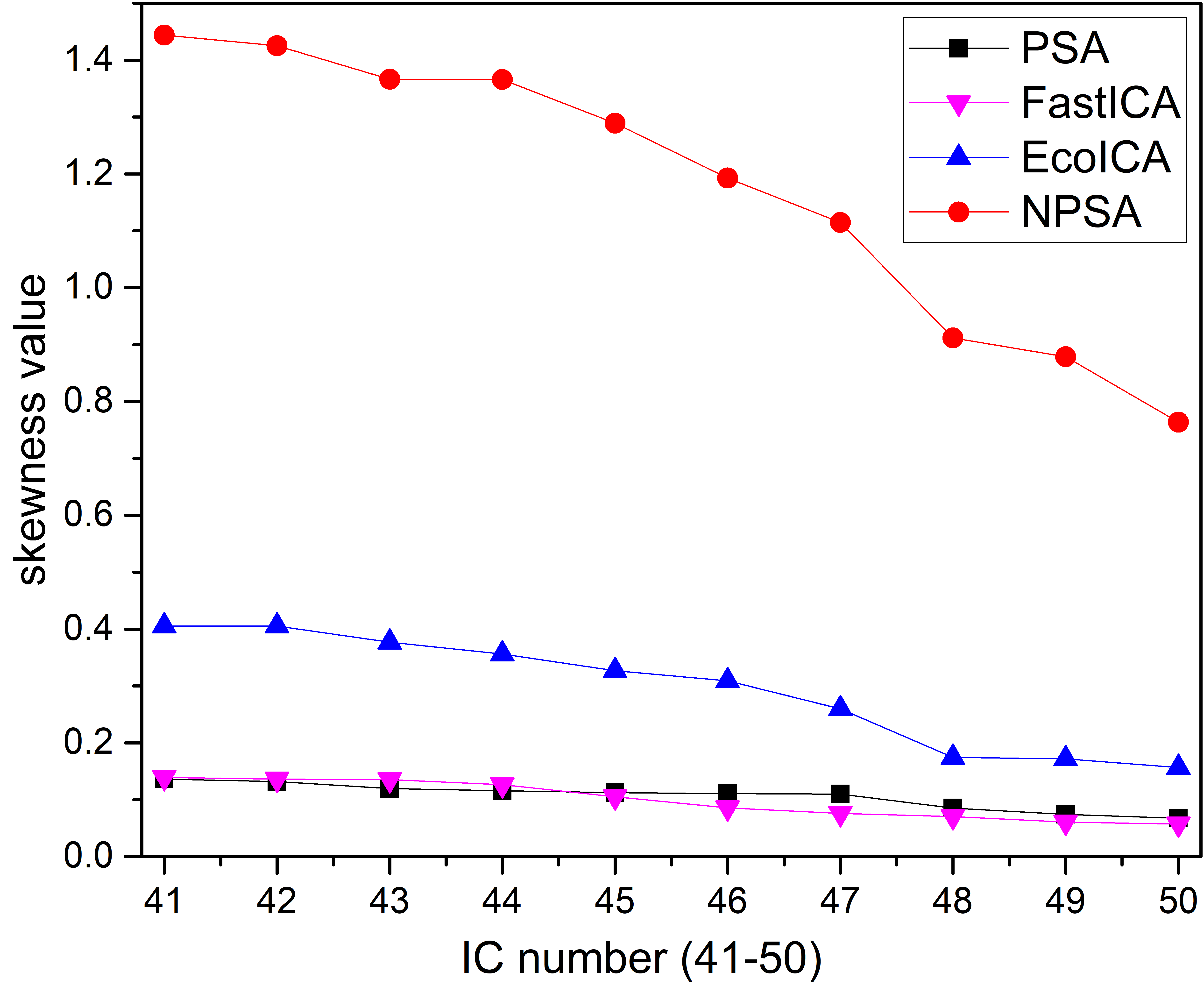}
  \label{sv5}
\end{subfigure}
\begin{subfigure}[htbp!]{0.23\textwidth}
  \centering
  \includegraphics[width=1\textwidth]{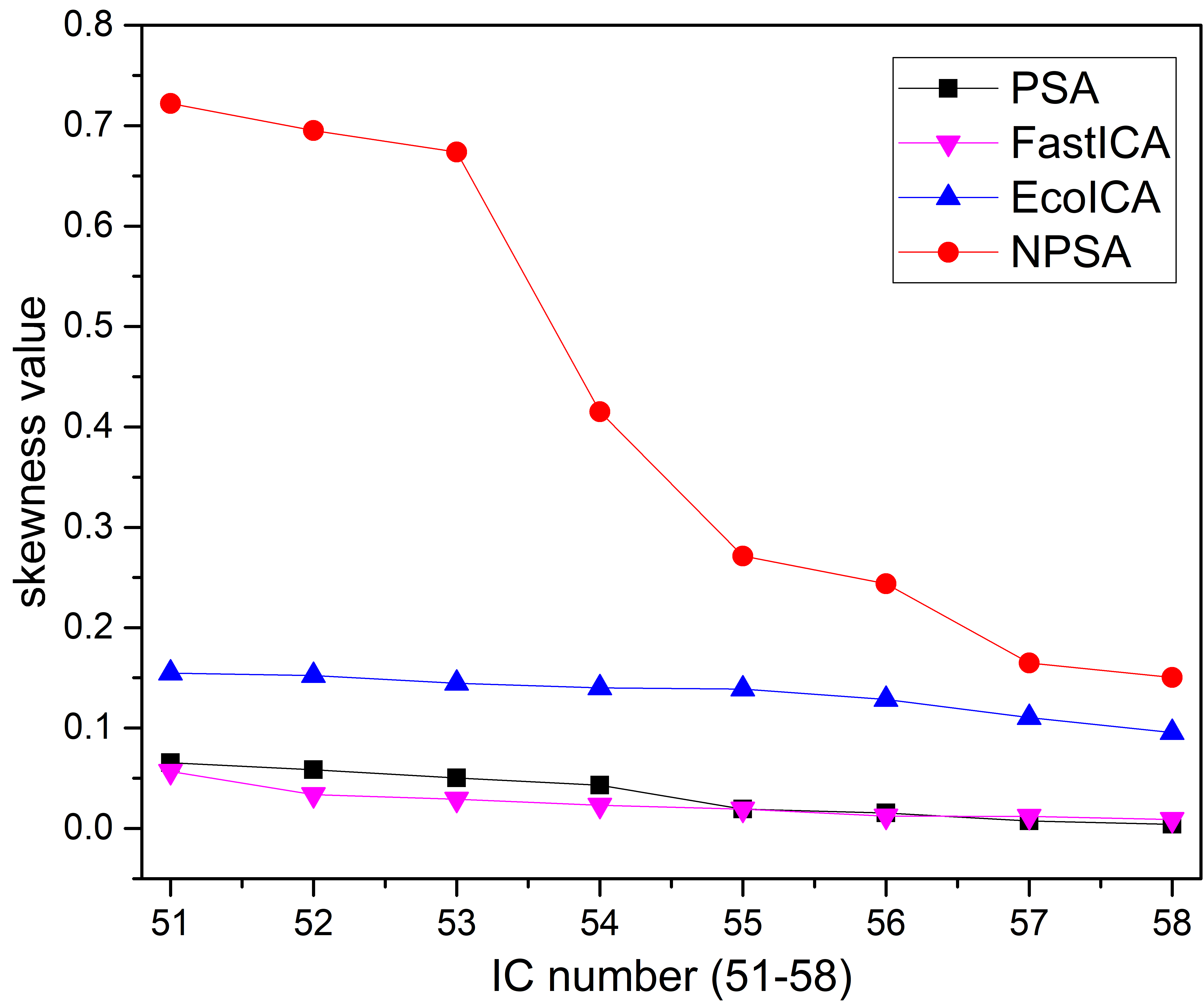}
  \label{sv6}
\end{subfigure}
  \caption{\quad Skewness  comparison of FastICA, EcoICA, PSA and NPSA for the hyperspectral image. }
  \label{hsivalue}
\end{figure}

\begin{figure}[t]
    \centering
  \includegraphics[width=0.3\textwidth]{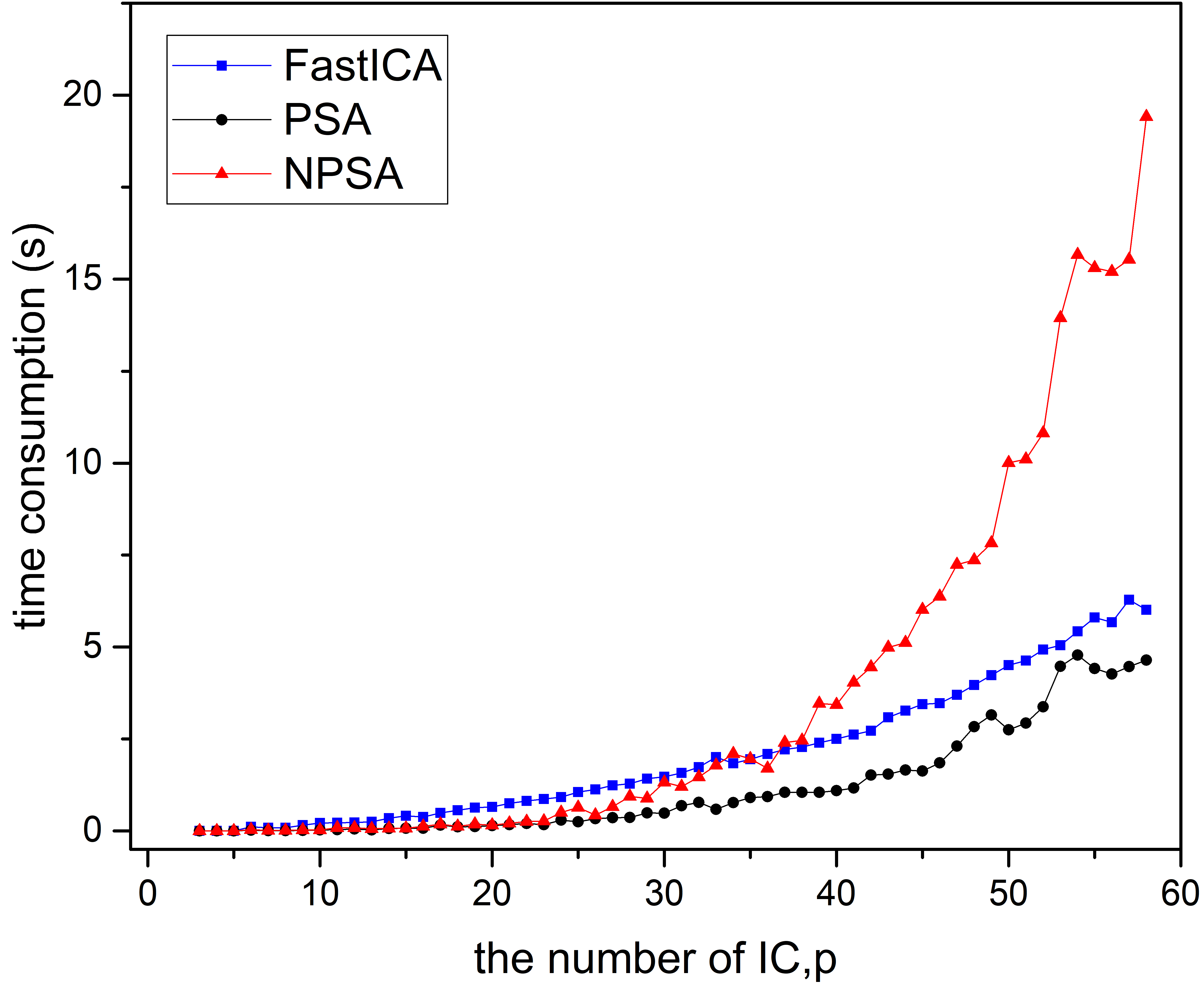}  \\
  \caption{\quad Time consumption comparison of FastICA, PSA and NPSA as a function of the number of ICs. }
  \label{hsitime}
\end{figure}
Besides, a time efficiency comparison is also conducted  in this experiment. The number of ICs, i.e, $p$, ranges from  $6 $ to $ 58$ and the time curve as a function of $p$ is plotted in Fig.~\ref{hsitime}. As $p$ increases, the time consumption of NPSA is greater than that of PSA. This is because we need to update the tensor in each unit of NPSA while this repeated computation in PSA can be simplified \cite{PSA}. However, the time of NPSA for a full-band data set, i.e., $p=58$, is about $ 19 $ seconds, which is still efficient and acceptable for the real hyperspectral image. Note that since the time of EcoICA for a full-band data set is about $ 60 $ seconds, which is much larger than that of the other three algorithms, so we do not take it into our comparison in this experiment.  To sum up, our algorithm can obtain a higher accuracy in  the sacrifice of  some time efficiency at an acceptable level.


\section{Conclusion}
Orthogonal complement constraint is a widely used  strategy to prevent the solution to be determined  from converging to the previous ones\cite{PSA,MPSA,fastICA}. However, such a  constraint can be  irreconcilably contradicted with the inherent nonorthogonality of supersymmetric tensor.
In this paper, originated from PSA,  we have proposed a new algorithm, which is named  nonorthogonal principal skewness analysis (NPSA). In NPSA, a more relaxed constraint than the orthogonal constraint is proposed  to search for the locally  maximum skewness direction in a larger space, and thus we can obtain the more accurate results. A detailed theoretical analysis is also presented to justify its validity. Furthermore, it is interesting to  find that the differences of PSA and NPSA lie in the order in  which they perform the orthogonal complement projection and the $n$-mode operation. We first apply the algorithm into the BIS probelm   and  several accuracy metrics evaluated in our  experiments show that NPSA can obtain a more accurate and robust result. Experiments for real multi/hyperspectral data also demonstrate that NPSA  outperforms the other algorithms in extracting the ICs of the image.

On the one hand, our method can be extended to fourth-or-higher-order case naturally. Both PSA and NPSA focus on the skewness index, which may not be always the best choice to describe the statistical structure of the data. Kurtosis and other indices  can be considered as the alternative. On the other hand, NPSA needs to update the coskewness tensor in each unit, which makes it slightly more time-consuming than PSA. So in the future, more efficient optimization methods will be worth studying.




\bibliographystyle{unsrt}
\bibliography{NPSA}

\end{document}